\pdfoutput=1
\documentclass[a4paper,fleqn]{cas-sc}

\usepackage[authoryear,longnamesfirst]{natbib}

\def\tsc#1{\csdef{#1}{\textsc{\lowercase{#1}}\xspace}}
\tsc{WGM}
\tsc{QE}
\usepackage{microtype}
\usepackage{booktabs} %

\usepackage{hyperref}

\usepackage{amsmath}
\usepackage{amssymb}
\usepackage{mathtools}
\usepackage{amsthm}

\newtheorem{theorem}{Theorem}
\newtheorem{lemma}{Lemma}
\newtheorem{corollary}{Corollary}

\theoremstyle{definition}

\theoremstyle{remark}
\newtheorem{rmk}{Remark}

\usepackage{tikz}
\usepackage{subcaption}
\usepackage{multirow}
\usepackage{pgfplots}
\pgfplotsset{compat=1.17}
\usepackage[utf8]{inputenc} %
\usepackage[T1]{fontenc}    %
\usepackage{hyperref}       %
\usepackage{booktabs}       %
\usepackage{amsfonts}       %
\usepackage{nicefrac}       %
\usepackage{microtype}      %
\usepackage{xcolor}         %
\usepackage{overpic}
\usepackage{tabularx}
\usepackage{booktabs}
\usepackage{makecell}
\newcommand{\R}{\mathbb{R}} %

\newcommand{\bxi}{\boldsymbol{\xi}}
\newcommand{\bxip}{\boldsymbol{\xi'}}

\newcommand{\bsa}{\boldsymbol{a}}
\newcommand{\bsb}{\boldsymbol{b}}

\newcommand{\bsx}{\boldsymbol{x}}
\newcommand{\bsxp}{\boldsymbol{x'}}
\newcommand{\bsw}{\boldsymbol{w}}

\ExplSyntaxOn
\keys_define:nn { cas / fig }
  { unknown .code:n = { \tl_if_blank:nT {#1} { \tl_set:Nx \l_fig_pos_tl { \l_keys_key_str } } } }
\ExplSyntaxOff

\ExplSyntaxOn
\cs_set:Npn \__first_footerline:
  { \group_begin: \small \sffamily \__short_authors: :~ { \rmfamily \itshape Preprint } \group_end: }
\ExplSyntaxOff

\begin{document}
\let\WriteBookmarks\relax

\shorttitle{Operator-informed initialization for FF-PINNs}    

\shortauthors{J. Molina et al.}  

\title [mode = title]{Operator-informed initialization for Fourier features physics-informed neural networks}

\author[1]{Juan Molina}[orcid=0009-0001-1106-0215]
\cormark[1]
\ead{juan.jose.molina@cenia.cl}
\author[2]{Paris Perdikaris}[]
\ead{pgp@seas.upenn.edu}
\author[1,3,4]{Mircea Petrache}[]
\ead{mpetrache@uc.cl}
\author[3]{Matías Courdurier}[]
\ead{mcourdurier@uc.cl}
\author[1,5,6]{Francisco Sahli~Costabal}[]
\ead{fsahlic@uc.cl}

\affiliation[1]{organization={National Center for Artificial Intelligence},
city={Santiago},
country={Chile}}

\affiliation[2]{organization={Department of Mechanical Engineering and Applied Mechanics, University of Pennsylvania}, city={Philadelphia}, state={PA}, country={USA}}

\affiliation[3]{organization={Department of Mathematics, Pontificia Universidad Cat\'olica de Chile},
city={Santiago},
country={Chile}}

\affiliation[4]{organization={Institute for Mathematical and Computational Engineering, Pontificia Universidad Cat\'olica de Chile},
city={Santiago},
country={Chile}}

\affiliation[5]{organization={Department of Mechanical and Metallurgical Engineering and Institute for Biological and Medical Engineering, Pontificia Universidad Cat\'olica de Chile},
city={Santiago},
country={Chile}}

\affiliation[6]{organization={Millennium Institute for Intelligent Healthcare Engineering, iHEALTH},
city={Santiago},
country={Chile}}

\cortext[1]{Corresponding author}

\begin{abstract}
Physics-Informed Neural Networks (PINNs) typically exhibit spectral bias, where some frequencies of the target function converge more slowly than others. In this work, we analyze the training dynamics of Fourier Feature PINNs in the Neural Tangent Kernel regime to address this limitation. We derive an explicit evolution equation to estimate the residual error in the frequency domain, demonstrating that the convergence rate of specific frequencies is primarily governed by the product of the differential operator's symbol and the spectral density of the initialization weights. Leveraging this theoretical insight, we propose an informative initialization strategy that tailors the initial weight distribution to the specific PDE being solved. With this method, we can diminish the operator-induced spectral bias, balancing the convergence rates across the frequency spectrum and achieving better prediction accuracy. Numerical experiments on linear and nonlinear partial differential equations confirm that this initialization strategy improves learning dynamics and approximation accuracy across frequencies compared to standard initialization methods, with no additional training cost.
\end{abstract}

\begin{keywords}
\sep Physics-Informed Neural Networks, 
\sep Frequency bias, 
\sep Spectral bias,
\sep Initialization,
\sep Neural network dynamics
\end{keywords}

\maketitle

\section*{Introduction}

Physics-Informed Neural Networks (PINNs) have emerged as a powerful framework for approximating solutions of partial differential equations (PDEs) by embedding the governing physical laws directly into the training objective \citep{raissi2019physics,luo2025physics}. By combining data-driven learning with the minimization of PDE residuals, PINNs provide a flexible alternative to classical numerical solvers. Their versatility has enabled successful applications across a broad range of scientific domains, including fluid mechanics \citep{raissi2020hidden,de2021assessing,bounnah2025physics} and bioengineering \citep{sahli2020physics,kissas2020machine}. Despite these achievements, PINNs often suffer from slow convergence and reduced predictive accuracy when the target solution exhibits sharp gradients, localized structures, or highly oscillatory behavior \citep{wang2021eigenvector,deshpande2022investigations}. Increasing evidence suggests that these difficulties are closely related to the intrinsic \emph{frequency bias} of neural networks.

Frequency Bias (FB), also known as spectral bias or the frequency principle, refers to the tendency of neural networks trained with gradient-based methods to learn low-frequency components of a target function earlier and faster than high-frequency ones \citep{xu2018understanding,luo2019theory,rahaman2019spectral}. This phenomenon has been observed across a wide range of architectures \citep{geifman2022spectral,bhattamishra2022simplicity} and optimization methods \citep{khodakarami2026spectral}, including quasi-Newton approaches such as BFGS \citep{ma2021frequency}. In the context of PDE learning, frequency bias can become a major limitation since many physically relevant solutions exhibit multiscale or broadband frequency content, as commonly encountered in fluid dynamics, wave propagation, and related applications. Consequently, understanding the mathematical mechanisms that govern frequency bias in PINNs is essential for improving both their efficiency and predictive capabilities.

\section{Background}

\textbf{Frequency Bias:} Theoretical investigations of FB began with the analysis of shallow neural networks with smooth activation functions such as $\tanh$ \citep{xu2018understanding}. Subsequent works studied the phenomenon through eigenvalue analyses and assumptions on data distributions \citep{ronen2019convergence,cao2021towards}, while others provided exact or asymptotic characterizations in the infinite-width regime \citep{luo2022exact,luo2022upper,zhang2021linear}. Beyond discrete spectral analyses based on kernel eigenvalues, continuous Fourier-domain descriptions of the learning dynamics have also been developed, enabling a more refined understanding of how individual frequency modes evolve during training \citep{molina2026optimizing,luo2022exact}. 

In the context of PINNs, frequency bias manifests in a more intricate manner due to the composite structure of the loss function, which combines PDE residual terms with boundary, initial, and data constraints. It has been observed that different components of the PINN loss may converge at substantially different rates when analyzed through their limiting NTK dynamics \citep{wang2022and}. Furthermore, frequency-dependent effects specific to PINNs have been reported in \citet{XU2025113905}.

\textbf{Fourier Feature Networks in PINNs:}
Several architectures have been proposed to mitigate the spectral bias of standard neural networks, including Fourier Feature (FF) networks, SIREN, Gaussian networks, and WIRE \citep{tancik2020fourier,sitzmann2020implicit,ramasinghe2022beyond,saragadam2023wire}. While all these approaches improve the representation of high-frequency components, FF networks are particularly attractive for the present work because they provide an explicit and controllable parametrization of the network's frequency content. In an FF network, the input coordinates are mapped through sinusoidal functions with frequencies sampled from a prescribed distribution. Consequently, the spectral properties of the representation can be modified directly through this distribution, without changing the subsequent network architecture. Moreover, FF mappings admit a natural interpretation in the NTK regime, where the frequency distribution controls the spectral properties and effective bandwidth of the kernel \citep{tancik2020fourier,molina2026optimizing}. This explicit connection between initialization, spectral representation, and training dynamics makes FF networks especially suitable for incorporating prior information about the frequency geometry induced by a PDE.

Within the PINN framework, FF networks have been successfully applied to a variety of scientific and engineering problems, including wind turbine power prediction \citep{aravanis2025fourier} and forward and inverse fluid mechanics problems \citep{sallam2023use}. Several FF-based variants have also been proposed to address multiscale and highly oscillatory PDEs. Examples include multiscale Fourier feature mappings for acoustic wave equations \citep{10630853}, iterative basis-enhancement strategies \citep{wu2025iterative}, finite-difference FF-PINNs for nonlinear fluid dynamics problems \citep{vasheghani2026novel}, and FF-PINNs for time-domain wave equations \citep{11045735,wu2026novel}. More sophisticated approaches include spatially adaptive Fourier encodings \citep{LIU2025106886}, natural frequency-based embeddings that align the representation spectrum with the physical frequency content of the underlying system \citep{wang2026structural}, Fourier feature-based domain decomposition methods \citep{LI2023114963}, and residual-based frequency modeling techniques \citep{HOU2026108247}.

From a functional perspective, FF networks can also be viewed as a particular instance of physics-informed kernel function neural networks \citep{fu2024physics}, where a fixed basis expansion is combined with a trainable neural architecture. This interpretation naturally connects FF-PINNs with kernel methods and motivates their analysis in the NTK regime.

\textbf{Neural Tangent Kernel Regime:} When the width of the hidden layers tends to infinity, neural networks trained by gradient descent admit a deterministic kernel description through the Neural Tangent Kernel (NTK) \citep{jacot2018neural,lee2019wide,arora2019exact,arora2019fine}. In this regime, the network output evolves according to a linear evolution equation in function space, enabling a precise characterization of the training dynamics and their spectral properties.

The NTK framework has been successfully employed to study PINN training dynamics \citep{wang2022and} and to analyze how architectural modifications influence spectral convergence rates \citep{wang2021eigenvector}. The NTK framework has become one of the principal tools for analyzing the learning dynamics of NNs (\citep{tancik2020fourier}).

\textbf{Related Work:} Several strategies have been proposed to mitigate frequency bias in PINNs through improved initialization procedures and frequency-aware training schemes. PDE-agnostic initialization methods have been introduced to stabilize and accelerate training independently of the underlying operator \citep{liu2022novel}. Other approaches first pretrain a neural network without physics constraints and subsequently transfer the learned weights to initialize the PINN, aiming to provide a more favorable spectral starting point for optimization \citep{tarbiyati2024weight}.

Beyond initialization, a variety of frequency-domain techniques have been developed to alleviate spectral imbalance. These include methods that transform high-frequency components into lower-frequency representations before training \citep{cai2020phase,cai2024batch}, Fourier feature-based domain decomposition approaches that reduce high-frequency content locally \citep{LI2023114963}, and architectures specifically designed to address frequency bias in PINNs \citep{XU2025113905}. Other works propose adaptive or spatially varying Fourier feature encodings that progressively expose higher frequencies during training \citep{LIU2025106886}.

Although these methods often demonstrate improved empirical performance, they are commonly tailored to particular PDE classes, architectures, or training settings, and typically rely on problem-specific or heuristic frequency manipulations. In contrast, our approach provides a unified framework in which the frequency-dependent dynamics are derived directly from the differential operator. This allows the same initialization principle to be applied across a broad range of linear and nonlinear PDEs by adapting only to their symbol, without redesigning the architecture or training procedure for each problem. Moreover, the resulting PDE-informed distribution is supported by a frequency-domain characterization of the residual dynamics, providing a more general and theoretically grounded mechanism for mitigating PDE-induced frequency bias.

\section*{Main Contributions}

\begin{itemize}
    \item We derive a theoretical framework describing the residual frequency dynamics of PINNs with 2-layer Fourier Feature networks in the NTK regime. The resulting dynamics explicitly reveal the interaction between the initial parameter density distribution and the Fourier symbol of the PDE linear operator with constant coefficients. Moreover, the convergence of the residual is strongly influenced by frequency regions where the PDE symbol is small or vanishes. This characterization provides a principled explanation of how the differential operator contributes to FB in PINNs.
    \item Motivated by the theoretical analysis, we propose a novel PDE-informed initialization strategy that adapts the sampling distribution of parameters of the first layer according to the geometry induced by the PDE symbol. 
    \item We validate the theoretical predictions through extensive numerical experiments on several PDEs. The proposed initialization consistently mitigates frequency bias, accelerates optimization, improves residual convergence, and reduces variability across random initialization seeds.
    \item We demonstrate that the proposed methodology remains effective beyond the theoretical assumptions used in the analysis, including nonlinear PDEs, deep networks, and ADAM optimization. Furthermore, we show that the proposed initialization can be integrated into existing PINN frameworks, such as JAX-PI, thereby improving performance without incurring additional computational cost during training.
\end{itemize}

\section{Problem setup}\label{sec:Setup}

In this section, we introduce the neural network model, the operator formulation of the PDE, and the corresponding training dynamics. The analysis builds upon the neural network and gradient-flow framework introduced in \citet{molina2026optimizing}, which we extend here to the PINN setting.

\subsection{Operator formulation of the PDE}

Let $D$ denote the spatial domain. For time-dependent problems, let $I=[0,\mathcal{T}]\subset\mathbb{R}$ be the time interval and define the space--time domain $\Omega = I\times D\subset\mathbb{R}^{d}$. For stationary problems, we simply take $\Omega=D\subset\mathbb{R}^{d}$. We denote by $u:\Omega\rightarrow\mathbb{R}$ the solution of the PDE written in operator form: \begin{equation} 
    \mathcal{D}_{\bsx} u(\bsx) = 0, \quad \bsx \in \Omega, 
\end{equation}
together with boundary and (if applicable) initial conditions 
\begin{align} 
    \mathcal{B}_{\bsx} u(\bsx) &= 0, \quad \bsx \in I \cap \partial D, \\
    \mathcal{T}_{\bsx} u(\bsx) &= 0, \quad \bsx \in \{\tau=0\}\cap\Omega.   
\end{align}
Here, $\mathcal{D}_{\bsx}$ denotes a differential operator, $\mathcal{B}_{\bsx}$ a boundary operator, and $\mathcal{T}_{\bsx}$ an initial-condition operator. We note that the vector $\bsx\in\mathbb{R}^{d}$ represents the spatial and temporal coordinates $\tau$.

In what follows, we focus on the contribution of the interior operator $\mathcal{D}_{\bsx}$ to the training dynamics, as it captures the main mechanism underlying the spectral behavior of the method. 

\subsection{Neural network model and training dynamics}

We approximate the solution $u$ by a neural network 
$f(\bsx,\Theta(t)):\Omega \times \mathbb{R}^{(d+1)\times m}\to\mathbb{R}$, where $\Theta(t)$ denotes the network parameters at continuous training time $t\in[0,T]$, with $T$ representing the final training time. We consider a single-hidden-layer neural network of width $m$:
\begin{equation}\label{eq:NNgeneral}
    f(\bsx,\Theta(t)) = \frac{1}{\sqrt{m}}\sum_{k=1}^{m} a_{k}(t)\, g\left(\bsw_{k}(t)\cdot \bsx\right),
\end{equation}
where $g:\mathbb{R}\to\mathbb{R}$ is a nonlinear activation function and 
$\Theta(t)=\{\bsa(t),\mathbf{W}(t)\}$ denotes the collection of output and hidden-layer parameters.

In practical PINN formulations, the network parameters are obtained by minimizing a loss function that combines the PDE residual with the constraints imposed by the boundary and initial conditions. Specifically, the pointwise PINN loss is defined as
\begin{align}\label{eq:PINNLoss}
    \mathcal{L}(\bsx,\Theta)
    =&\;\mathcal{L}_{\mathrm{pde}}(\bsx,\Theta)+\beta_{\mathrm{bc}}\mathcal{L}_{\mathrm{bc}}(\bsx,\Theta)+\beta_{\mathrm{ic}}\mathcal{L}_{\mathrm{ic}}(\bsx,\Theta)\nonumber\\
    :=&\;\left(\mathcal{D}_{\bsx}f(\bsx,\Theta)-\mathcal{D}_{\bsx}u(\bsx)\right)^2 + \beta_{\mathrm{bc}}\left(\mathcal{B}_{\bsx}f(\bsx,\Theta)-\mathcal{B}_{\bsx}u(\bsx)\right)^2 + \beta_{\mathrm{ic}}\left(\mathcal{T}_{\bsx}f(\bsx,\Theta)-\mathcal{T}_{\bsx}u(\bsx)\right)^2,
\end{align}
where $\beta_{\mathrm{bc}},\beta_{\mathrm{ic}}>0$.

Now, as a first step, we look at the PDE loss $\mathcal{L}_{\mathrm{pde}}$. Considering a uniformly distributed sampling in the interior of $\Omega$, the expected empirical risk is given by
\begin{equation}\label{eq:risk}
    \mathcal{R}_S(\Theta)=\frac{1}{|\Omega|}\int_\Omega\left(\mathcal{D}_{\bsx}f(\bsx',\Theta)-\mathcal{D}_{\bsx}u(\bsx')\right)^2d\bsx'.
\end{equation}
The training dynamics is approximated by gradient flow, $\frac{d}{dt}\Theta(t) = -\nabla_\Theta \mathcal{R}_S(\Theta(t)),$ then the expected evolution of the network can be described as
\begin{align}
    \frac{d}{dt}f(\bsx,\Theta(t))  &=\frac{d} {dt}\left(f(\bsx,\Theta(t))-u(\bsx)\right) = \nabla_\Theta f(\bsx,\Theta(t)) \cdot \frac{d}{dt}\Theta(t) \nonumber \\
    &=-\frac{1}{|\Omega|}\int_\Omega \nabla_\Theta f(\bsx,\Theta(t)) \cdot \nabla_\Theta \big(\mathcal{D}_{\bsx} f(\bsx',\Theta(t))\big)\, \mathcal{D}_{\bsx} \left(f(\bsx',\Theta(t)) - u(\bsx')\right)d\bsx'.
\end{align}
Finally, by defining the residual $r$ and tangent kernel $K_m$ as
\begin{equation}
    r(\bsx,\Theta(t)) := f(\bsx,\Theta(t)) - u(\bsx) \quad \text{and} \quad K_m(\bsx,\bsx';t) := \nabla_\Theta f(\bsx,\Theta(t)) \cdot \nabla_\Theta f(\bsx',\Theta(t));
\end{equation}
assuming sufficient regularity, and neglecting boundary terms, we obtain
\begin{equation}
    \frac{d}{dt} r(\bsx,\Theta(t)) = - \int_{\Omega} K_m(\bsx,\bsx';t)\, \big(\mathcal{D}_{\bsx'}^* \mathcal{D}_{\bsx'} r\big)(\bsx',\Theta(t))\, d\mu(\bsx'),
\end{equation}
where $\mathcal{D}_{\bsx'}^*$ denotes the adjoint of $\mathcal{D}_{\bsx'}$ and $d\mu(x)=\frac{1}{|\Omega|}dx$ is the density of the data sampling distribution.
\begin{rmk}
    The expression above should include a boundary term. Such boundary term, together with the  boundary and initial conditions in the full PINN's loss function participate in driving the dynamics of the training. We conjecture that in all cases the term above is one of the main driving directions.
\end{rmk}

\subsection{Linearized residual dynamics}

We consider the Neural Tangent Kernel (NTK) regime \citep{jacot2018neural, lee2019wide, arora2019exact}, assuming (i) The parameters are initialized i.i.d. from distributions $\rho_{\bsw}$ and $\rho_a$ with zero mean and finite variance ($\sigma_{\bsw}^2$ and $\sigma_a^2$ respectively) and (ii) The width $m \to \infty$, the kernel converges to
\begin{equation}
K(\bsx,\bsx') := \mathbb{E}_{\theta \sim \rho_a \times \rho_{\bsw}}\left[\nabla_\theta f^0(\bsx,\theta)\cdot \nabla_\theta f^0(\bsx',\theta)\right],
\end{equation}
where $f^0(\bsx,\theta)=a\,g(\bsw\cdot \bsx)$. Therefore, replacing $K_m$ by $K$, we obtain
\begin{equation}\label{eq:residualDynamics}
\frac{d}{dt} r(\bsx,t)
= - \int_{\Omega} K(\bsx,\bsx')\, \big(\mathcal{D}_{\bsx'}^* \mathcal{D}_{\bsx'} r\big)(\bsx',t)\, d\mu(\bsx').
\end{equation}
This equation generalizes the classical NTK residual dynamics by introducing the operator $\mathcal{D}_{\bsx'}^*\mathcal{D}_{\bsx'}$, which acts as a differential preconditioner and fundamentally alters the spectral properties of the training dynamics.

\section{Frequency-domain formulation}\label{sec:FBeq}

To rigorously describe the frequency dynamics, we work in the framework of tempered distributions. Let $\mathcal{S}(\mathbb{R}^d)$ denote the Schwartz space and $\mathcal{S}'(\mathbb{R}^d)$ its dual. For $h \in \mathcal{S}'(\mathbb{R}^d)$, the Fourier transform is defined in the distributional sense by
\begin{equation}
\langle \mathcal{F}[h], \phi \rangle = \langle h, \mathcal{F}[\phi] \rangle, 
\quad \forall \phi \in \mathcal{S}(\mathbb{R}^d),
\end{equation}
where $\widehat{\phi}(\bxi) = \int_{\mathbb{R}^d} \phi(\bsx)e^{-2\pi i \bsx \cdot \bxi}d\bsx$. We denote by $\widehat{h} = \mathcal{F}[h]$ the Fourier transform of $h$ and $\bxi$ the frequency variable.

We emphasize that this work considers the following assumptions:
\begin{itemize}
    \item The differentiation operator $\mathcal{D}_{\bsx}$ is linear with constant coefficients. In this hypothesis, a key role is played by the Fourier symbol of the operator $\mathcal{D}_{\bsx}^*\mathcal{D}_{\bsx}$. Recall that the Fourier transform converts differentiation into multiplication in the frequency domain. In particular, for a sufficiently regular function $r$, $\mathcal{F}[\partial_{x_j}r](\bxi)=2\pi\mathrm{i}\xi_j\widehat{r}(\bxi),$ and higher-order derivatives generate higher powers of the frequency variables. Consequently, a linear differential operator with constant coefficients is transformed into a multiplication operator whose multiplier is a polynomial of the frequency variable. In the case of $\mathcal{D}_{\bsx}^*\mathcal{D}_{\bsx}$, the symbol is non-negative and describes the frequency-dependent response of the PDE operator. Denoting this symbol by $p:\mathbb{R}^d\to\mathbb{R}_+$, we have
    \begin{equation}
        \mathcal{F}\big[(\mathcal{D}_{\bsx}^*\mathcal{D}_{\bsx} r)(\cdot,\Theta)\big](\bxi)
        =
        p(\bxi)\widehat{r}(\bxi,\Theta).
    \end{equation}
    \item The approximation of the empirical loss by an integral over $\Omega$ in equation \eqref{eq:risk} corresponds to assuming that the sampling of the data is sufficiently dense and approximately uniform over the domain. This assumption is standard in PINNs, where the data are typically defined on regular grids, but it may break down in the presence of highly non-uniform collocation strategies or irregular geometries. For example, problems involving domains with sharp corners, boundary layers, or moving obstacles.
\end{itemize}
Computing the Fourier transform in equation \eqref{eq:residualDynamics} of the residual dynamics, derived in Section~\ref{sec:Setup}, we obtain the following result.

\begin{lemma}
{\bf [proof in Appendix \ref{app:Dynamics}]\label{thm:Lemma}} Under the assumptions of Section~\ref{sec:Setup} and Section~\ref{sec:FBeq}, the frequency dynamics of the residual, ignoring the contribution from the boundary terms, can be written in the sense of distributions, as
\begin{equation}\label{eq:frequency-general}
\partial_t \widehat{r}\left(\bxi,\Theta(t)\right)
= - \int_{\mathbb{R}^d} \widehat{K}(\bxi,\bxip)\, p(\bxip)\, \widehat{r}\left(\bxip,\Theta(t)\right)\, d\bxip,
\end{equation}
where the kernel $\widehat{K}$ is defined by
\begin{equation}
\widehat{K}(\bxi,\bxip)
= \mathbb{E}_{\theta \sim \rho_a \times \rho_{\bsw}}
\left[
\overline{\mathcal{F}_{\bsx' \to \bxip}\big[\nabla_\theta f^0(\bsx',\theta)\big]}
\cdot
\mathcal{F}_{\bsx \to \bxi}\big[\nabla_\theta f^0(\bsx,\theta)\big]
\right].
\end{equation}
\end{lemma}

The proof consists of taking the Fourier transform of the residual dynamics in equation \eqref{eq:residualDynamics} and using Parseval's identity. This is analogous to \citet{luo2022upper} and details are provided in Appendix~\ref{app:Dynamics}.

Equation \eqref{eq:frequency-general} shows that the symbol of the operator $\mathcal{D}_{\bsx}^*\mathcal{D}_{\bsx}$, the multiplicative factor $p(\bxip)$, modulates the spectral behavior of the training dynamics. Computing $\hat K$ more explicitly, we obtain the following.

\begin{theorem}{\bf [proof in Appendix \ref{app:Dynamics}]}
\label{thm:frequency-dynamics}
Under the assumptions of Sections~\ref{sec:Setup} and \ref{sec:FBeq}, and assuming $g \in H^1(\mathbb{R})$, the frequency dynamics \eqref{eq:frequency-general} can be expressed as
\begin{align}\label{eq:frequency-dynamics}
\left\langle \partial_t \widehat{r}, \psi \right\rangle
= & -\frac{\sigma_a^2}{4\pi^2}\int_{\mathbb{R}^d}\overline{\nabla \psi(\bxi)} \cdot \int_{\mathbb{R}^2}\nabla\Big( p\Big(\frac{y'}{y}\bxi\Big) \widehat{r}\Big(\frac{y'}{y}\bxi,\Theta\Big)\Big) \rho_{\bsw}\Big(\frac{1}{y}\bxi\Big) \frac{\overline{\widehat{g'}(y')} \widehat{g'}(y)}{|y|^d}\, dy' dy \, d\bxi \nonumber \\
& - \int_{\mathbb{R}^d}\overline{\psi(\bxi)} \int_{\mathbb{R}^2} p\Big(\frac{y'}{y}\bxi\Big) \widehat{r}\Big(\frac{y'}{y}\bxi,\Theta\Big) \rho_{\bsw}\Big(\frac{1}{y}\bxi\Big) \frac{\overline{\widehat{g}(y')} \widehat{g}(y)}{|y|^d} \, dy' dy \, d\bxi.
\end{align}
\end{theorem}

We now specialize the previous result to the two-layer neural network introduced in Section~\ref{sec:Setup}.

\section{Frequency dynamics for Fourier Feature networks}

We now consider the two-layer FF network
\begin{equation}\label{eq:FFnetwork}
f(\bsx,\Theta(t)) =
\frac{1}{\sqrt{2m}}\sum_{k=1}^{m} a_k(t)g_1(\bsw_k(t)\cdot \bsx)
+
\frac{1}{\sqrt{2m}}\sum_{k=1}^{m} b_k(t)g_2(\bsw_k(t)\cdot \bsx),
\end{equation}
where $g_1=\cos(2\pi \cdot)$, $g_2=\sin(2\pi \cdot)$ and $\Theta(t)=\{\bsa(t),\bsb(t),\mathbf{W}(t)\}$.

In such case, equation \eqref{eq:frequency-dynamics} takes a closer and more interpretable expression. Using the previous analysis, we obtain the following explicit frequency dynamics.

\begin{corollary}{\bf [proof in Appendix \ref{app:Dynamics}]}\label{cor:Nogeneral_PINNs}
Under the NTK assumptions, if $\mathbb{E}[a_k(0)^2]=\mathbb{E}[b_k(0)^2]=\sigma_a^2$ and $\rho_{\bsw}$ is even, then the frequency dynamics, ignoring the contribution from boundary terms, is given by
\begin{equation}\label{eq:Nogeneral_PINNs}
\frac{d}{dt}\widehat{r}(\bxi,\Theta(t)) = \sigma_a^2 \,\mathrm{div}\Big(\rho_{\bsw}(\bxi)\nabla \big(p(\bxi)\widehat{r}(\bxi,\Theta(t))\big)\Big) - p(\bxi)\rho_{\bsw}(\bxi)\widehat{r}(\bxi,\Theta(t)),
\end{equation}
with $\bxi \in \R$, $t>0$, and initial condition
\begin{equation}
    \widehat{r}_0(\bxi) = \widehat{r}(\bxi,\Theta(0)) = \mathcal{F}_{\bsx \rightarrow \bxi} \left[ f(\bsx,\Theta(0))-u(\bsx) \right].
\end{equation}
\end{corollary}

This equation reveals that, for FF networks, the training dynamics in frequency space take the form of a diffusion-reaction equation modulated by the symbol $p(\bxi)$ and the initialization distribution $\rho_{\bsw}(\bxi)$.

We remark that in the FF setting, the operator governing the dynamics \eqref{eq:Nogeneral_PINNs}, is a local operator with a relatively weaker cross-frequency influence in the training.

\subsection{Extension to full PINNs losses}

The previous analysis focuses on the PDE residual contribution to the PINN objective. In practical applications, however, the loss function also incorporates boundary and initial conditions (equation \ref{eq:PINNLoss}). Consequently, following the same derivation as before, the expected full residual dynamics will satisfy
\begin{equation}
    \partial_t r = -\int_\Omega K \Big(  \mathcal{D}_{\bsx}^{*}\mathcal{D}_{\bsx}r +    \beta_{\mathrm{bc}}    \mathcal{B}_{\bsx}^{*}\mathcal{B}_{\bsx}r +   \beta_{\mathrm{ic}}    \mathcal{T}_{\bsx}^{*}\mathcal{T}_{\bsx}r \Big)d\mu  +   \mathcal{R}_{\partial\Omega},
\end{equation}
where $\mathcal{R}_{\partial\Omega}$ collects the boundary contributions arising from integration by parts.

The analysis in Section~\ref{sec:FBeq} focuses in the contribution to the dynamics from the term $\mathcal{D}_{\bsx}^{*}\mathcal{D}_{\bsx}$.
In contrast, the boundary and initial operators generally give rise to trace operators whose Fourier representations are nonlocal. Consequently, the frequency dynamics are expected to take the schematic form
\begin{align}\label{eq:FFResidualDynamicsFinal}
    \partial_t\widehat r = & \;  \sigma_a^2\operatorname{div}    \Big(\rho_{\bsw}\nabla\big(p_{\mathrm{pde}}\widehat r +   \beta_{\mathrm{bc}}    \mathcal K_{\mathrm{bc}}[\widehat r]   +    \beta_{\mathrm{ic}}    \mathcal K_{\mathrm{ic}}[\widehat r] \big) \Big) 
    - \rho_{\bsw} \Big(    p_{\mathrm{pde}}\widehat r  +    \beta_{\mathrm{bc}}  \mathcal K_{\mathrm{bc}}[\widehat r]  +    \beta_{\mathrm{ic}}    \mathcal K_{\mathrm{ic}}[\widehat r]\Big)  
    +    \widehat{\mathcal R}_{\partial\Omega},
\end{align}
where $p_{\mathrm{pde}}$ denotes the Fourier symbol of $\mathcal{D}_{\bsx}^{*}\mathcal{D}_{\bsx}$, $\mathcal K_{\mathrm{bc}}$ and $\mathcal K_{\mathrm{ic}}$ are linear nonlocal operators induced by the boundary and initial conditions, and $\widehat{\mathcal R}_{\partial\Omega}$ denotes the Fourier transform of the boundary contributions generated by Green's identity.

\subsection{Example: heat equation in one spatial dimension}

We illustrate the previous result with the heat equation $\partial_\tau u - \partial_{xx} u = 0,$ in the domain $\bsx = (\tau,x) \in [0,1] \times [-1,1]$, with Dirichlet boundary conditions and an initial condition. The differential operator is
\begin{equation}
    \mathcal{D}_{\bsx} = \partial_\tau - \partial_{xx} \quad \text{and} \quad \mathcal{D}_{\bsx}^*\mathcal{D}_{\bsx} = -\partial_{\tau\tau} + \partial_{xxxx}.
\end{equation}
Taking the Fourier transform in $(\tau,x)$, we obtain the operator symbol
\begin{equation}
    p_{\mathrm{pde}}(\bxi)    = (2\pi \xi_\tau)^2 + (2\pi \xi_x)^4, \quad \text{with } \bxi = (\xi_\tau,\xi_x).
\end{equation}

\paragraph{Boundary operator:}

For Dirichlet boundary conditions on $x=\pm 1$, the operator $\mathcal{B}_{\bsx}^*\mathcal{B}_{\bsx}$ acting on the residual $r$ is given by
\begin{equation}
    \mathcal{B}_{\bsx}^*\mathcal{B}_{\bsx} r = \delta(x-1)\, r(\tau,1) +  \delta(x+1)\, r(\tau,-1).
\end{equation}
Taking the Fourier transform in $(\tau,x)$ and expressing the boundary traces in terms of $\widehat{r}$, one obtains
\begin{equation}
    \widehat{\mathcal{B}_{\bsx}^*\mathcal{B}_{\bsx} r}=2 \int_{\mathbb{R}} \cos\big(2\pi(\xi_x-\xi_x')\big)\widehat{r}(\xi_\tau,\xi_x')\, d\xi_x'.
\end{equation}
Therefore, the boundary operator induces a convolution in the spatial frequency variable, coupling all modes through the kernel $\cos(2\pi\cdot)$.

\paragraph{Initial operator:}

Similarly, the initial condition operator $\mathcal{T}_{\bsx}$ corresponds to a trace at $\tau=0$, leading to another nonlocal operator
\begin{equation}
    \widehat{\mathcal{T}_{\bsx}^*\mathcal{T}_{\bsx} r} = \int_{\mathbb{R}} \widehat{r}(\xi_\tau',\xi_{\bsx})\, d\xi_\tau'.
\end{equation}
The term $\widehat{\mathcal R}_{\partial\Omega}$ is less explicit.

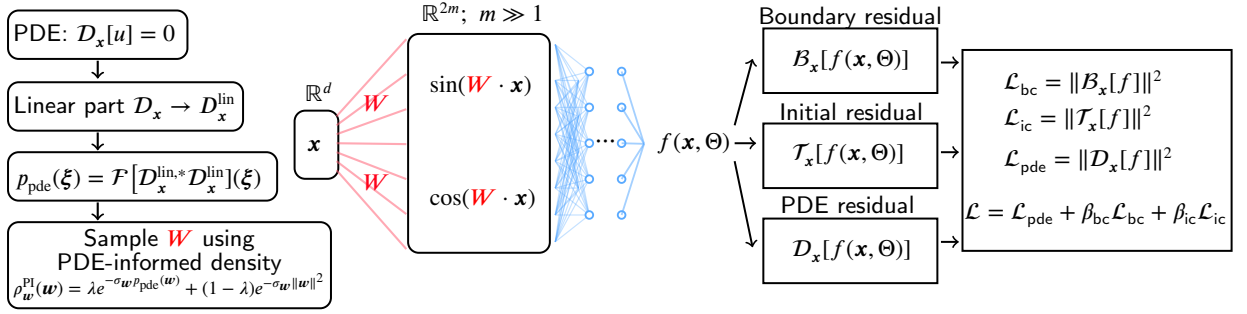
\begin{figure}[t]
    \centering                
    
\tikzset{every picture/.style={line width=0.75pt}} %

\begin{tikzpicture}         [x=0.75pt,y=0.75pt,yscale=-.9,xscale=.9]
    \draw[rounded corners=4pt]
    (0,66) rectangle (100,93);
    
    \draw (48,80)
    node {\small PDE: $\mathcal{D}_{\bsx}[u]=0$};

    \draw[->,thick]
    (50,94) -- (50,104);
    
    \draw[rounded corners=4pt]
    (0,106) rectangle (131,133);
    
    \draw (65,120)
    node {\small Linear part $\mathcal{D}_{\bsx}\rightarrow\mathcal \mathcal{D}_{\bsx}^{\rm lin}$};

    \draw[->,thick]
    (50,134) -- (50,144);
    
    \draw[rounded corners=4pt]
    (0,145) rectangle (150,172);
    
    \draw (73,160)
    node {\small $p_{\rm pde}(\bxi)=\mathcal{F}\big[\mathcal{D}_{\bsx}^{\rm {lin},*}\mathcal{D}_{\bsx}^{\rm lin}](\bxi)$};

    \draw[->,thick]
    (50,173) -- (50,183);
    
    \draw[rounded corners=4pt]
    (0,184) rectangle (180,232);
    
    \draw (90,194)
    node {\small Sample $\textcolor{red}{W}$ using};
    \draw (90,207)
    node {\small PDE-informed density};
    
    \draw (90,220)
    node {\scriptsize
    $\rho_{\bsw}^{\rm PI}({\bsw})
    =
    \lambda e^{-\sigma_{\bsw} p_{\rm pde}({\bsw})}
    +
    (1-\lambda)e^{-\sigma_{\bsw}\|{\bsw}\|^2}$
    };

    \draw  [rounded corners=4pt, color={rgb, 255:red, 0; green, 0; blue, 0 }  ,draw opacity=1 ] (159.72,121.83) -- (184.09,121.83) -- (184.09,159.83) -- (159.72,159.83) -- cycle ;
    
    \draw[
        rounded corners=4pt,
        color={rgb,255:red,0;green,0;blue,0},
        draw opacity=1
    ]
    (223,79.07) -- (302,79.07) -- (302,201.07) -- (223,201.07) -- cycle ;

    \foreach \yy in {125,130,135,140,145,150,155}
    {
        \draw[
            color={rgb,255:red,255; green,100; blue,120},
            line width=0.8pt,
            opacity=0.55
        ]
        (183.09,\yy)
        --
        (222,{82 + (\yy-125)*3.9});
    }

    \foreach \yy in {85,105,120,135,150,165,180,195}
    {
        \foreach \zz in {105,125,145,165}
        {
            \draw[
                color={rgb,255:red,90;green,170;blue,255},
                line width=0.45pt,
                opacity=0.35
            ]
            (305,\yy) -- (318,\zz);
        }
    }

    \foreach \yy in {85,105,120,135,150,165,180,195}
    {
        \foreach \zz in {100,120,140,160,180}
        {
            \draw[
                color={rgb,255:red,90;green,170;blue,255},
                line width=0.45pt,
                opacity=0.35
            ]
            (305,\yy) -- (324,\zz);
        }
    }
    
    \foreach \yy in {100,120,140,160,180}
    {
        \filldraw[
            fill=white,
            draw={rgb,255:red,90;green,170;blue,255},
            line width=0.7pt
        ]
        (324,\yy) circle (1.6pt);
    }
    
    \draw (333,140) node {$\cdots$};
    
    \foreach \yy in {100,120,140,160,180}
    {
        \filldraw[
            fill=white,
            draw={rgb,255:red,90;green,170;blue,255},
            line width=0.7pt
        ]
        (342,\yy) circle (1.6pt);
    }
    
    \foreach \yy in {100,120,140,160,180}
    {
        \draw[
            color={rgb,255:red,90;green,170;blue,255},
            line width=0.7pt,
            opacity=0.55
        ]
        (342,\yy) -- (354.5,140.33);
    }

    \draw (165,136.4)
    node [anchor=north west][inner sep=0.75pt]
    {$\boldsymbol{x}$};
    
    \draw (196,111.4)
    node [anchor=north west][inner sep=0.75pt]
    [color={rgb,255:red,255;green,2;blue,27}]
    {$W$};
    
    \draw (196,155.4)
    node [anchor=north west][inner sep=0.75pt]
    [color={rgb,255:red,255;green,2;blue,27}]
    {$W$};
    
    \draw (233.5,100.4)
    node [anchor=north west][inner sep=0.75pt]
    {$\sin(\textcolor{red}{W}\cdotp\boldsymbol{x})$};
    
    \draw (233.5,163.4)
    node [anchor=north west][inner sep=0.75pt]
    {$\cos(\textcolor{red}{W}\cdotp\boldsymbol{x})$};
    
    \draw (360,132.4)
    node [anchor=north west][inner sep=0.75pt]
    {$f(\boldsymbol{x},\Theta)$};
    
    \draw (164,104.4)
    node [anchor=north west][inner sep=0.75pt]
    {$\mathbb{R}^{d}$};
    
    \draw (229,59.4)
    node [anchor=north west][inner sep=0.75pt]
    {$\mathbb{R}^{2m};\ m\gg1$};

    \draw[->] (405,130) -- (415,93);
    \draw[->] (405,140) -- (417,140);
    \draw[->] (405,150) -- (415,198);
    
    \draw  [color={rgb,255:red,0;green,0;blue,0}] 
    (421,182) rectangle (518,218);
    
    \draw (432,200) node [anchor=west] {$\mathcal{D}_{\bsx} [f(\boldsymbol{x},\Theta)]$};
    \draw (468,174) node {PDE residual};

    \draw  [color={rgb,255:red,0;green,0;blue,0}] 
    (421,76) rectangle (518,112);
    
    \draw (432,94) node [anchor=west] {$\mathcal{B}_{\bsx}[f(\boldsymbol{x},\Theta)]$};
    \draw (469.5,70) node {Boundary residual};

    \draw  [color={rgb,255:red,0;green,0;blue,0}] 
    (421,129) rectangle (518,165);
    
    \draw (432,146) node [anchor=west] {$\mathcal{T}_{\bsx}[f(\boldsymbol{x},\Theta)]$};
    \draw (470,122) node {Initial residual};

    \draw[->] (520,95) -- (530,95);
    \draw[->] (520,145) -- (530,145);
    \draw[->] (520,195) -- (530,195);
    
    \draw  [color={rgb,255:red,0;green,0;blue,0}] 
    (532,88) rectangle (682,200);
    
    \draw (550,150) node [anchor=west] {$\mathcal{L}_{\text{pde}}=\|\mathcal{D}_{\bsx}[f]\|^2$};
    \draw (550,106) node [anchor=west] {$\mathcal{L}_{\text{bc}}=\|\mathcal{B}_{\bsx}[f]\|^2$};
    \draw (550,126) node [anchor=west] {$\mathcal{L}_{\text{ic}}=\|\mathcal{T}_{\bsx}[f]\|^2$};
    \draw (528,180) node [anchor=west] {$\mathcal{L}=\mathcal{L}_{\text{pde}} + \beta_{\text{bc}} \mathcal{L}_{\text{bc}} + \beta_{\text{ic}} \mathcal{L}_{\text{ic}}$};
\end{tikzpicture}
\caption{Pipeline of the proposed PDE-informed initialization. The governing PDE determines the principal symbol $p_{\mathrm{pde}}$, which is used to sample the Fourier frequencies $W$ through the proposed initialization density. The sampled frequencies define the Fourier Feature embedding of the PINN, whose parameters are trained by minimizing the PDE, boundary, and initial-condition residual losses.}
\label{fig:FFmodel}
\end{figure}

\section{Qualitative analysis of the residual frequency dynamics}
\label{sec:DesignPropose}

To gain insight into the role of the initialization distribution, we first consider the frozen Fourier Features setting, in which the first-layer weights $\mathbf W$ remain fixed at their initial values. Under this assumption, the diffusion term in equation \eqref{eq:FFResidualDynamicsFinal} vanishes, and the residual frequency dynamics reduces to
\begin{align}\label{eq:FreezeDynamics}
    \frac{d}{dt}\widehat{r}(\bxi,\Theta(t)) = & - \rho_{\bsw}(\bxi)\Big(p_{\mathrm{pde}}(\bxi)\widehat{r}(\bxi,\Theta)
    + \beta_{\mathrm{bc}}\mathcal{K}_{\mathrm{bc}}[\widehat{r}](\bxi)
    + \beta_{\mathrm{ic}}\mathcal{K}_{\mathrm{ic}}[\widehat{r}](\bxi)\Big)
    + \widehat{\mathcal R}_{\partial\Omega}
    [\widehat{K},\widehat{r}](\bxi).
\end{align}
Although simplified, this setting remains closely related to the full dynamics, as NTK theory predicts that, in the wide-network regime, the first-layer parameters remain close to their initialization throughout training \citep{yang2021tensor}. Similar behavior has also been observed empirically for Fourier Feature networks \citep{tancik2020fourier}.

Equation~\eqref{eq:FreezeDynamics} highlights the central role of the initialization density $\rho_{\bsw}$. In the frozen-feature regime, $\rho_{\bsw}(\bxi)$ acts as a frequency-dependent scaling of the dynamics: frequencies receiving little probability mass at initialization evolve slowly or, in the limiting case, not at all. Therefore, the initialization directly influences which frequency components can be efficiently represented and learned.

The PDE symbol $p_{\mathrm{pde}}$ introduces an additional frequency-dependent modulation through the interior residual term. Frequencies for which $p_{\mathrm{pde}}(\bxi)$ is small receive only a weak contribution from the PDE residual, whereas frequencies associated with larger values of $p_{\mathrm{pde}}(\bxi)$ are strongly constrained by the interior equation. Importantly, $p_{\mathrm{pde}}$ characterizes only the interior PDE contribution. When $p_{\mathrm{pde}}(\bxi)=0$, this contribution vanishes in \eqref{eq:FreezeDynamics}, so the evolution of the corresponding frequency components depends entirely on the remaining boundary, initial-condition, and boundary-integral terms. Thus, the interior differential operator identifies spectral directions that are weakly constrained by the PDE residual, while the auxiliary conditions provide the additional information required to determine the physically admissible solution \citep{evans2010partial}.

This interpretation can also be expressed through the quadratic response of the interior residual loss to Fourier perturbations. Let $\sigma_{\mathcal D}(\bxi)$ denote the Fourier symbol of the linear constant-coefficient operator $\mathcal D_{\bsx}$, $\mathcal L_{\mathrm{pde}}(v)=\frac{1}{2}\left|\mathcal D_{\bsx}v\right|_{L^2}^{2}$, and a Fourier perturbation $v(\bsx)=a e^{2\pi i\bxi\cdot\bsx}$. Since $\mathcal D_{\bsx}v=\sigma_{\mathcal D}(\bxi)a e^{2\pi i\bxi\cdot\bsx}$, its quadratic contribution to the loss is proportional to
\begin{equation}
    |\sigma_{\mathcal D}(\bxi)|^2 |a|^2 = p_{\mathrm{pde}}(\bxi)|a|^2.
\end{equation}
Equivalently, the second variation of the interior residual loss along this Fourier direction is weighted by $p_{\mathrm{pde}}(\bxi)$. Consequently, frequencies for which $p_{\mathrm{pde}}(\bxi)$ is small correspond to weakly penalized, low-curvature directions of the interior residual loss, while the set $\{p_{\mathrm{pde}}=0\}$ defines flat directions of the interior residual loss in this idealized Fourier description. This is consistent with equation \eqref{eq:FreezeDynamics}, where the same factor $p_{\mathrm{pde}}(\bxi)$ controls the strength of the interior residual dynamics.

Figure~\ref{fig:Psymbol} illustrates this behavior empirically. The largest residual errors tend to persist around regions where $p_{\mathrm{pde}}$ is small, i.e., precisely where the interior PDE residual provides the weakest frequency-dependent control. This observation should therefore not be interpreted as implying that the interior term dominates the dynamics in those regions; rather, the interior symbol determines their spectral geometry, while the boundary and initial-condition contributions remain essential for determining the physically admissible solution. Additional evidence is provided in Appendix~\ref{app:transport_symbol_mismatch}, where we show that the training dynamics progressively deteriorate when the polynomial used to construct the PDE-informed initialization is increasingly misaligned with the symbol of the true differential operator.

The PDE symbol also carries information about the spectral geometry of the solution. Consider a homogeneous linear constant-coefficient PDE on $\mathbb{R}^d$, $\mathcal{D}_{\bsx}u=0$. Let $\sigma_{\mathcal D}(\bxi)$ denote the Fourier symbol of $\mathcal D_{\bsx}$. Then $\sigma_{\mathcal D}(\bxi)\widehat u(\bxi)=0$ in the sense of distributions. Since $p_{\mathrm{pde}}(\bxi)=|\sigma_{\mathcal D}(\bxi)|^2$, we define
\begin{equation}
\label{eq:p=0}
    \mathcal C_{\mathrm{pde}}:=\{\bxi:p_{\mathrm{pde}}(\bxi)=0\}
    =\{\bxi:\sigma_{\mathcal D}(\bxi)=0\}.
\end{equation}
It follows that $\operatorname{supp}\widehat u \subseteq \mathcal C_{\mathrm{pde}}$. Thus, in this idealized whole-space homogeneous setting, $\mathcal C_{\mathrm{pde}}$ is relevant not only to the residual dynamics but also to the spectral support of the solution.

On a bounded domain, this exact support relation is broadened by spatial localization. Indeed, if $\chi$ denotes a window restricting the solution to the computational domain, then $\widehat{\chi u} = \widehat{\chi}*\widehat u$, so an ideal spectrum supported on $\mathcal C_{\mathrm{pde}}$ is spread into a neighborhood of this set. In the presence of a source, $\mathcal D_{\bsx}u=h$, the whole-space relation becomes
\begin{equation}
    \widehat u(\bxi)= \frac{\widehat h(\bxi)} {\sigma_{\mathcal D}(\bxi)}, \qquad \sigma_{\mathcal D}(\bxi)\neq0,
\end{equation}
showing that low-symbol regions may amplify spectral components of the source.

The transport equation provides a particularly transparent example. Consider $u_\tau+c u_x=0$ on $\mathbb{R}^2$. Its solutions have the form $u(\tau,x)=f(x-c\tau)$. Taking the Fourier transform in both variables gives, up to the normalization convention of the Fourier transform,
\begin{equation}
    \widehat{u}(\xi_\tau,\xi_x) = \widehat{f}(\xi_x) \delta(\xi_\tau+c\xi_x).
\end{equation}
Hence, the spectrum is supported on the characteristic line $\xi_\tau+c\xi_x=0$, which coincides with the zero set of $    p_{\mathrm{pde}}(\xi_\tau,\xi_x) =4\pi^2(\xi_\tau+c\xi_x)^2$. 
A standard isotropic Gaussian initialization does not exploit this characteristic geometry.

\begin{figure*}[t]
    \centering
    \includegraphics[width=0.95\textwidth]{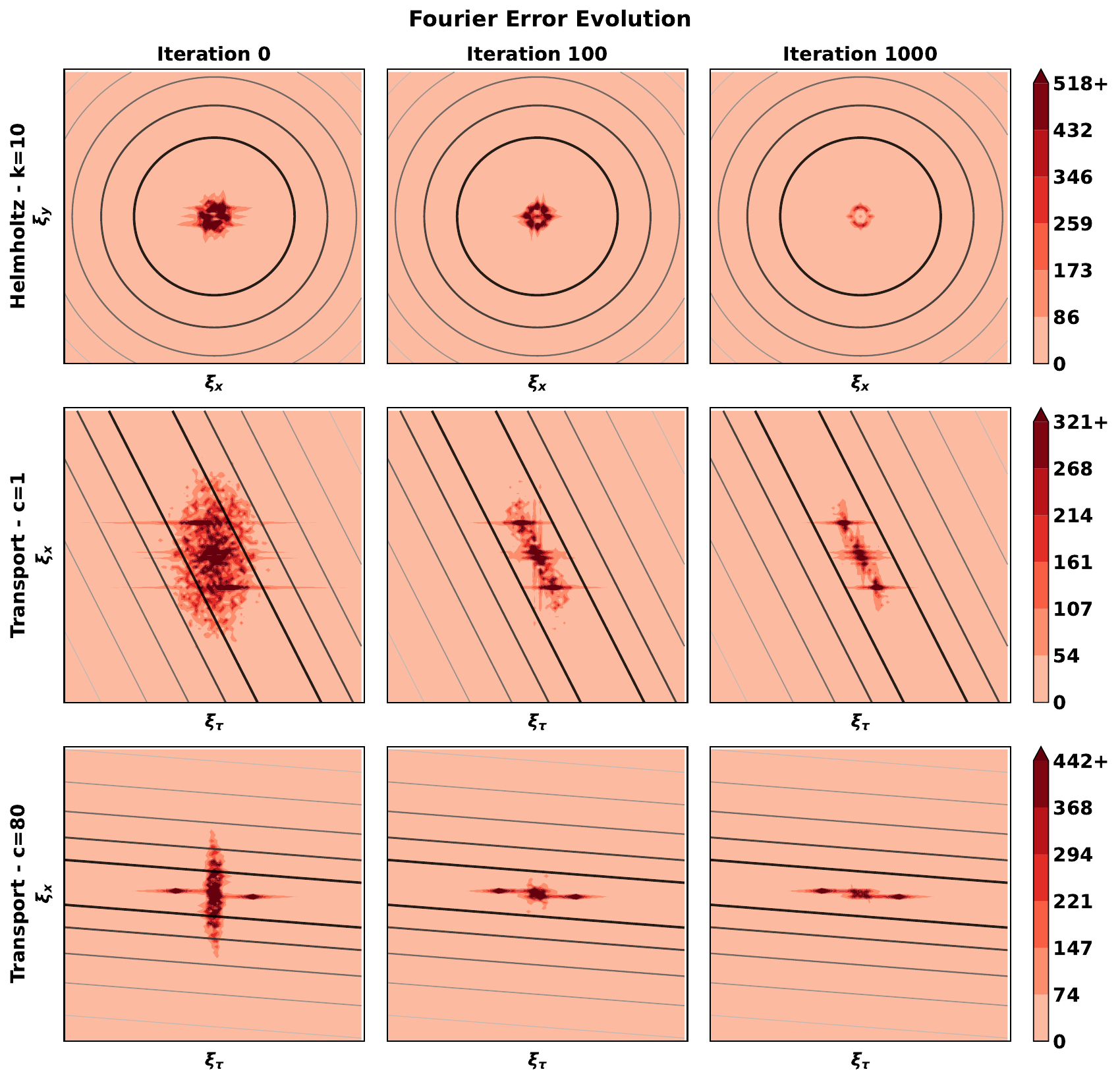}
    \caption{Frequency-domain (modulus) error dynamics for three PDEs solved with PINNs. From top to bottom: the 2D Helmholtz equation $u_{xx}+u_{yy}+k^2u=0$, with $k=10$, and the 1D transport equation $u_\tau+c u_x=0$ with $c=1$ and $c=80$. The corresponding polynomials are $p_{\mathrm{pde}}=(4\pi^2(\xi_x^2+\xi_y^2)-k^2)^2$ for the Helmholtz equation and $p_{\mathrm{pde}}=4\pi^2(\xi_\tau+c\xi_x)^2$ for the transport equation. The results show that the largest residual errors tend to persist in frequency regions where $p_{\mathrm{pde}}$ is small or vanishes, consistently with the weak interior control predicted by equation~\eqref{eq:FreezeDynamics}.}\label{fig:Psymbol}
\end{figure*}

\subsection{PDE-informed initialization to mitigate the frequency bias}

The analysis of the previous section shows that the low-level sets of $p_{\mathrm{pde}}$ play a distinguished role from two complementary perspectives. First, they identify frequency directions that are weakly constrained by the interior PDE residual and therefore exhibit slow interior dynamics. Second, for homogeneous linear constant-coefficient equations, the zero set of $p_{\mathrm{pde}}$ characterizes the regions in frequency space on which the solution may be supported. Motivated by these approximation and optimization considerations, we propose a family of PDE-informed initialization distributions that increases the probability of sampling Fourier frequencies in these regions. Specifically, we define
\begin{equation}\label{eq:rhoPI_general}
    \rho_{\bsw}^{\mathrm{PI}}(\bsw;\sigma_{\bsw},\lambda,k)
    = \frac{\lambda}{C_1} e^{-\sigma_{\bsw}p_{\mathrm{pde}}(\bsw)}     + \frac{1-\lambda}{C_2} e^{-k\,\sigma_{\bsw}|\bsw|^2},
\end{equation}
where $\sigma_{\bsw}>0$ controls the concentration of the distribution and $\lambda\in[0,1]$ determines the relative weight of the PDE-informed component. The normalization constants
\begin{equation}
    C_1 = \int_{\mathbb{R}^{d}}e^{-\sigma_{\bsw}p_{\mathrm{pde}}(\bsw)},d\bsw,
    \qquad C_2 = \int_{\mathbb{R}^{d}}e^{-k\,\sigma_{\bsw}|\bsw|^2},d\bsw,
\end{equation}
ensure that both components define valid probability densities.

The proposed initialization pipeline is summarized in Figure~\ref{fig:FFmodel}. Starting from the governing PDE, we identify its principal linear operator, namely, the linear differential operator containing the highest-order derivative terms of the equation \citep{evans2010partial}. For linear PDEs, this corresponds to the differential operator itself, whereas for the semilinear problems considered here it corresponds to the linear highest-order part of the PDE. We then compute the associated symbol $p_{\mathrm{pde}}$ and use it to construct the initialization density from which the Fourier frequencies $\mathbf W$ are sampled. These frequencies define the Fourier Feature embedding of the PINN, which is subsequently trained using the standard loss containing the PDE, boundary, and initial-condition residuals.

The PDE-informed component in equation \eqref{eq:rhoPI_general} provides a smooth concentration around the low-level sets of $p_{\mathrm{pde}}$. In this way, it increases the representational density precisely in directions for which the interior PDE residual provides weak control and which, for linear homogeneous problems, are also associated with the characteristic spectral geometry of the solution. The Gaussian component preserves broader isotropic coverage of frequency space. The latter is particularly useful for nonlinear problems, for which nonlinear interactions may generate spectral components not directly characterized by the symbol of the principal linear operator.

For linear PDEs, the symbol $p_{\mathrm{pde}}$ completely characterizes the frequency action of the differential operator. In this setting, we employ the pure PDE-informed initialization
\begin{equation}\label{eq:rhoPI_linear}
    \rho_{\bsw}^{\mathrm{PI}}(\bsw;\sigma_{\bsw},1)=\frac{e^{-\sigma_{\bsw}p_{\mathrm{pde}}(\bsw)}}{C_1},
\end{equation}
corresponding to $\lambda=1$ in
equation~\eqref{eq:rhoPI_general}.

For nonlinear PDEs, however, the symbol associated with the principal linear operator provides only a partial description of the frequency dynamics. Nonlinear interactions may generate additional frequencies that are not captured by $p_{\mathrm{pde}}$. In long training regimes, we observed that a pure PDE-informed initialization may therefore become overly specialized. To retain broader spectral coverage, we use equation~\eqref{eq:rhoPI_general} with $\lambda=0.5$, balancing the PDE-informed and Gaussian components.

If our sole objective were to compensate directly for the multiplicative factor $p_{\mathrm{pde}}$ appearing in the interior residual dynamics, a natural formal alternative would be an inverse-symbol weighting. For instance, one could consider $\rho_{\bsw}^{\mathrm{PI}}(\bsw;\sigma_{\bsw},\epsilon) = 1/(\sigma_{\bsw}p_{\mathrm{pde}}(\bsw)+\epsilon)$ for $|\bsw|<1$, and $\rho_{\bsw}^{\mathrm{PI}}=1$. Such a choice is more directly related to compensating the factor $p_{\mathrm{pde}}$ in the interior convergence rate. However, it introduces significant practical difficulties. In particular, the density becomes highly sensitive to $\epsilon$ near the zeros of $p_{\mathrm{pde}}$ and may become excessively concentrated in these regions, reducing the diversity of sampled frequencies. In our experiments, this inverse formulation consistently produced worse performance than the proposed exponential construction.

\begin{rmk}
    Throughout this work $\sigma_{\bsw}$ is used as the name for two independent hyperparameters. When considering a Gaussian initialization, $\sigma_{\bsw}$ denotes the standard deviation of the distribution $\mathcal{N}(0,\sigma_{\bsw}^{2})$, whereas in equation \eqref{eq:rhoPI_general}, $\sigma_{\bsw}$ is a concentration parameter rather than the standard deviation of $\rho_{\bsw}^{\mathrm{PI}}$. In both cases $\sigma_{\bsw}$ is the primary hyperparameter that controls the initialization distribution and has a decisive influence on the optimization dynamics.
\end{rmk}

\begin{figure}[t]
    \centering
    \includegraphics[width=\linewidth]{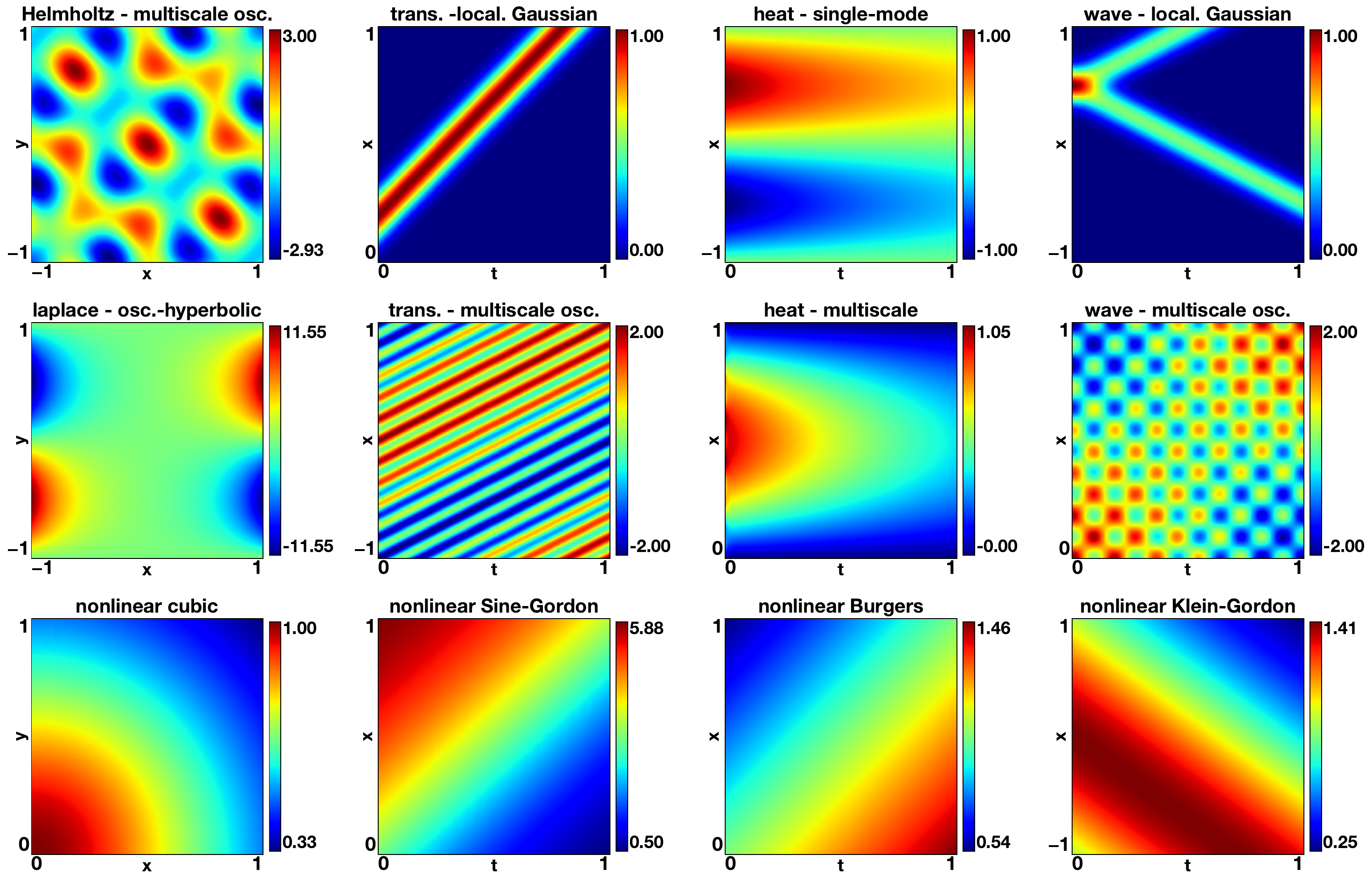}
    \caption{Visualization of the exact solutions listed in Table~\ref{tab:examples}. The examples cover a diverse range of spatial structures and spectral characteristics representative of the benchmark PDEs considered in this work.}
    \label{fig:GTs}
\end{figure}

\begin{table}[h]
    \centering
    \caption{Benchmark PDEs used for validation. The exact solutions are selected to exhibit diverse spectral characteristics representative of each PDE. The computational domains are $\Omega_1=[-1,1]^2$, $\Omega_2=[0,1]^2$, and $\Omega_3=[0,1]\times[-1,1]$. For simplicity, multiplicative constants in the polynomial $p_{\mathrm{pde}}$ are omitted.}\label{tab:examples}
\renewcommand{\arraystretch}{0.9}
\begin{tabularx}{\textwidth}{
>{\centering\arraybackslash}m{2.6cm}
>{\centering\arraybackslash}m{2.3cm}
>{\raggedright\arraybackslash}X
}
\toprule
PDE & $p_{\mathrm{pde}}$ & Exact solutions \\
\midrule

\makecell[c]{
\textbf{Transport} $(c=1)$ \\
$u_\tau+c\,u_x=0$}
& $(\xi_\tau+c\xi_x)^2$
&
\begin{tabular}{@{}l@{}}
\textbf{i) Multiscale oscillatory:}
{\small $u(\tau,x)=\cos(11\pi(x-c\tau))+\cos(\pi(x-c\tau))$ in $\Omega_3$}\\[0.4em]
\textbf{ii) Localized Gaussian:}
{\small $u(\tau,x)=e^{-100(x-c\tau-0.2)^2}$ in $\Omega_2$}
\end{tabular}
\\

\midrule

\makecell[c]{
\textbf{Heat} $(\alpha=0.1)$ \\
$u_\tau-\alpha u_{xx}=0$}
& $\xi_\tau^2+4\pi^2\alpha^2\xi_x^4$
&
\begin{tabular}{@{}l@{}}
\textbf{i) Single-mode:}
{\small $u(\tau,x)=e^{-\pi^2\alpha\tau}\sin(\pi x)$ in $\Omega_3$}\\[0.4em]
\textbf{ii) Multiscale:}
{\small $u(\tau,x)=e^{-\pi^2\alpha\tau}
\sin(\pi x)+\frac{1}{11}e^{-121\pi^2\alpha\tau}\sin(11\pi x)$ in $\Omega_2$}
\end{tabular}
\\

\midrule

\makecell[c]{
\textbf{Wave} $(c=1)$ \\
$u_{\tau\tau}-c^2u_{xx}=0$}
& $(\xi_\tau^2-c^2\xi_x^2)^2$
&
\begin{tabular}{@{}l@{}}
\textbf{i) Multiscale oscillatory:}
{\small $u(\tau,x)=\cos(\pi\tau)\cos(\pi x)
+\cos(11\pi\tau)\cos(11\pi x)$ in $\Omega_2$}\\[0.4em]
\textbf{ii) Traveling Gaussian:}
{\small $u(\tau,x)=\frac12\left(
e^{-50(x-c\tau-\frac12)^2}
+e^{-50(x+c\tau-\frac12)^2}
\right)$ in $\Omega_3$}
\end{tabular}
\\

\midrule

\makecell[c]{
\textbf{Laplace} \\
$-\Delta u = 0$}
& $(\xi_x^2+\xi_y^2)^2$
&
\textbf{i) Oscillatory-hyperbolic:}
{\small $u(x,y)=\sin(\pi x)\sinh(\pi y)$ in $\Omega_1$}
\\

\midrule

\makecell[c]{
\textbf{Helmholtz} ($k=10$) \\
$\Delta u + k^2u= 0$}
& $(4\pi^2(\xi_x^2+\xi_y^2)-k^2)^2$
&
\textbf{i) Multiscale oscillatory:}
{\small $u(x,y)=\cos(kx)+\cos(ky)+\cos\left(\frac{\sqrt{2}k}{2}(x+y)\right)$ in $\Omega_1$}
\\

\midrule

\makecell[c]{
\textbf{Nonlinear Laplace} \\
$\Delta u+u^3=0$}
& $(\xi_x^2+\xi_y^2)^2$
&
{\small $u(x,y)=\dfrac{1}{1+x^2+y^2}$ in $\Omega_2$}
\\

\midrule

\makecell[c]{
\textbf{Burgers} $(\nu=1)$ \\
$u_\tau+uu_x-\nu u_{xx}=0$}
& $\xi_\tau^2+4\pi^2\nu^2\xi_x^4$
&
{\small $u(\tau,x)=1-\tanh\!\left(\dfrac{x-\tau}{2\nu}\right)$ in $\Omega_2$}
\\

\midrule

\makecell[c]{
\textbf{Sine--Gordon} \\
$u_{\tau\tau}-u_{xx}+\sin(u)=0$}
& $(\xi_\tau^2-\xi_x^2)^2$
&
$u(\tau,x)=
4\arctan\!\left(
\exp\!\left(\frac{-(x-0.9\tau)}{\sqrt{0.19}}\right)
\right)$ in $\Omega_2$
\\

\midrule

\makecell[c]{
\textbf{Klein--Gordon} \\
$u_{\tau\tau}-u_{xx}+u-u^3=0$}
& $(\xi_\tau^2-\xi_x^2)^2$
&
{\small
$u(\tau,x)=
\dfrac{\sqrt{2}}
{\cosh\!\left(x+\sqrt{2}\tau\right)}
$ in $\Omega_3$}
\\

\bottomrule
\end{tabularx}
\end{table}

\section{Numerical Experiments}

\subsection{Training remarks}

The initial network parameters are sampled from the probability density functions $\rho_{a} \sim \mathcal{N}(0,\sigma_a^2)$, for the parameter $a$, and for the Fourier frequencies $W$ they are either sampled from the standard Gaussian initialization (GI) $\rho^{\mathrm{GI}}_{\bsw} \sim \mathcal{N}(0,\sigma_{\bsw}^2)$ or from the proposed PDE-informed initialization $\rho^{\mathrm{PI}}_{\bsw}$ defined in equation \eqref{eq:rhoPI_general}. In the case of $\rho^{\mathrm{PI}}_{\bsw}$, the frequencies are sampled numerically using rejection sampling over a bounded frequency domain (see Remark 3).

The two-layer neural networks are trained with a width $m=2000$ and without freezing the first layer. The collocation points are chosen as a uniform discretization of $100\times100$ points over the corresponding domain $\Omega_i$ ($i=1,2,3$). The weights $\beta_{\mathrm{bc}}$ and $\beta_{\mathrm{ic}}$ in the loss function are chosen to balance the gradients at initialization, and are kept fixed throughout the training. More precisely, their value is chosen such that
\begin{equation}
    \|\nabla_{\Theta} \mathcal{L}_{\mathrm{pde}}(\Theta(0))\|
    =\|\beta_{\mathrm{bc}}\nabla_{\Theta} \mathcal{L}_{\mathrm{bc}}(\Theta(0))\|
    =\|\beta_{\mathrm{ic}}\nabla_{\Theta} \mathcal{L}_{\mathrm{ic}}(\Theta(0))\|.
\end{equation}

Unless otherwise stated, training is performed without batching. Moreover, all figures reporting training dynamics are obtained from $10$ different random seeds. The solid line represents the median performance, while the shaded region indicates the range between the best and worst runs. All models are implemented and trained using JAX \citep{jax2018github}.

\begin{rmk}
     A practical consideration is the choice of the bounded frequency domain used for rejection sampling. Throughout this work, we use the Nyquist frequency associated with the computational grid, as it represents the largest frequency that can be learned without aliasing. However, this bound is determined solely by the spatial discretization and does not necessarily reflect the actual spectral content of the PDE solution. For example, the multiscale oscillatory solution of the wave equation contains frequencies only up to 5.5, whereas the corresponding Nyquist frequency is approximately 50. Consequently, the rejection sampling procedure allocates probability over a frequency range that is substantially larger than necessary. Restricting the sampling domain to a frequency range better matched to the solution could improve the efficiency of the proposed initialization by concentrating Fourier features in the relevant spectral region. Nevertheless, the Nyquist frequency provides a simple, problem-independent choice that proved sufficiently adequate for all the experiments considered in this work.
\end{rmk}

\begin{table*}[tb]
    \centering
    \small
    \caption{Learning rates used in the experiments and Optuna search intervals for the proposed PDE-informed initialization. The examples correspond to Figure~\ref{fig:SGD2layer}.}
    \label{tab:optuna_hyperparameters_sgd}
\begin{tabular}{|c|c|c|c|c|c|c|}
\hline
example &
$\eta_{\mathrm{GI}}$ &
$\sigma_{\bsw}^{\mathrm{GI}}$ &
GI range &
$\eta_{\mathrm{PI}}$ &
$\sigma_{\bsw}^{\mathrm{PI}}$ &
PI range \\
\hline

Helmholtz Multiscale Osc.
& $5\cdot10^{-7}$
& $1.42$
& $[0.3,3]$ 
& $5\cdot10^{-4}$
& $0.028$
& $[0.005, 0.5]$\\
\hline

Laplace Osc.--Hyperbolic
& $8\cdot10^{-7}$
& $1.4$
& $[0.5,5]$ 
& $2\cdot10^{-6}$
& $7.3\cdot10^{-5}$
& $[5\cdot10^{-5},5\cdot10^{-4}]$\\
\hline

Trans. Local. Gaussian
& $10^{-5}$
& $3.0$
& $[0.1,5]$ 
& $2\cdot10^{-3}$
& $0.62$
& $[5\cdot10^{-2},5]$\\
\hline

Trans. Multiscale Osc.
& $10^{-6}$
& $2.1$
& $[1,3]$ 
& $8\cdot10^{-5}$
& $2.6$
& $[0.1,10]$\\
\hline

Heat Single-Mode
& $5\cdot10^{-5}$
& $0.34$
& $[0.1,5]$ 
& $5\cdot10^{-5}$
& $7.7\cdot10^{-3}$
& $[10^{-4},10^{-2}]$\\
\hline

Heat Multiscale
& $6\cdot10^{-5}$
& $0.25$
& $[0.1,1]$ 
& $6\cdot10^{-4}$
& $0.78$
& $[0.1,1]$\\
\hline

Wave Local. Gaussian
& $5\cdot10^{-8}$
& $3.9$
& $[0.6,10]$ 
& $2\cdot10^{-5}$
& $6.7\cdot10^{-3}$
& $[10^{-5},10^{-2}]$\\
\hline

Wave Multiscale Osc.
& $7\cdot10^{-9}$
& $1.7$
& $[2,4]$ 
& $8\cdot10^{-6}$
& $7.6\cdot10^{-3}$
& $[10^{-5},10^{-2}]$\\
\hline

Nonlinear Cubic
& $5\cdot10^{-7}$
& $0.64$
& $[0.5,5]$ 
& $5\cdot10^{-6}$
& $3.5\cdot10^{-3}$
& $[5\cdot10^{-4},5\cdot10^{-3}]$\\
\hline

Nonlinear Burgers
& $8\cdot10^{-5}$
& $0.11$
& $[0.05,0.5]$ 
& $3\cdot10^{-5}$
& $7.47$
& $[5,50]$ \\
\hline

Nonlinear Sine--Gordon
& $8\cdot10^{-7}$
& $0.3$
& $[0.1,1]$ 
& $3\cdot10^{-5}$
& $0.037$
& $[5\cdot10^{-3},5\cdot10^{-2}]$ \\
\hline

Nonlinear Klein--Gordon
& $2\cdot10^{-6}$
& $0.26$
& $[0.05,5]$ 
& $2\cdot10^{-5}$
& $4.9\cdot10^{-3}$
& $[8\cdot10^{-5},0.008]$\\
\hline

\end{tabular}
\end{table*}

\subsubsection{Dependence of the learning rate on the initialization weights.}

In general, the stability of the gradient descent optimization used in the training of a neural network depends on the appropriate choice of the learning rate, which needs to be compatible with the loss function and the magnitude of its gradients.
This implies a key observation for Fourier Feature PINNs since the structure of the loss function will depend on the initialization of the network parameters: the appropriate learning rate for the PINNs training 
will depend on the distribution used to sample the frequencies $\bsw_k$ (the magnitude of the parameters $a_k$ do not produce changes in the loss function beyond a scaling factor).

To illustrate this phenomenon, consider the one-dimensional Laplace equation $-u_{xx}(x)=h(x),$ $x\in\Omega,$ where $h$ is a known source term. Suppose that the solution is approximated using a Fourier Feature network (equation \eqref{eq:FFnetwork}) where the frequencies are initialized according to $w_k \sim \mathcal N(0,\sigma_{w}^{2}).$ Taking derivatives with respect to $x$ yields
\begin{equation}
    f_{xx}(x, \Theta)=-\sum_{k=1}^m \left(a_k w_k^2 \sin(w_k x) + b_k w_k^2 \cos(w_k x)\right),
\end{equation}
thus, each Fourier mode is amplified by a factor $w_k^2$ under the action of the Laplace operator. Consequently, increasing the initialization variance $\sigma_{w}^{2}$ increases the probability of sampling high frequencies, which in turn produces larger values in the PDE residual. Using a standard squared residual loss (equation \eqref{eq:risk}), the residual scales proportionally to $w_k^2$, while the curvature of the loss scales approximately as $\lambda_{\max}\propto\max_k w_k^4,$ where $\lambda_{\max}$ denotes the dominant eigenvalue of the loss Hessian.
Stability of the gradient descent scheme requires sufficiently small learning rates whenever the Hessian eigenvalues become large, which does not scale proportionally to the gradients of the loss function. This example shows that the learning rate and the Fourier Feature initialization distribution should not be selected independently and it is not just a proportionality or power relation.

Such coupling becomes particularly relevant in the experiments presented in this paper, where we compare the standard initialization $\rho^{\mathrm{GI}}_{\bsw}$ with the proposed PDE-informed initialization $\rho^{\mathrm{PI}}_{\bsw}$. Since the effective frequency distributions differ significantly between the two initializations, the optimal learning rate will differ in each case.

\subsection{Training dynamics and solution accuracy}

A small empirical PINN loss does not necessarily imply an accurate approximation of the underlying PDE solution. The loss measures violations of the PDE and auxiliary constraints at the selected collocation points, whereas the relative $L^2$ error measures the discrepancy with the reference solution over the computational domain. This mismatch can lead to inaccurate or spurious solutions despite apparently successful optimization \citep{wang2026pinns}.

The geometry of the region in frequency space identified in Section~\ref{sec:DesignPropose} also helps explain why favorable training dynamics do not necessarily imply accurate predictions. In particular, modes near the zero-symbol set $\mathcal{C}_{\mathrm{pde}}$ (see equation \ref{eq:p=0}) are weakly constrained by the interior residual, while modes on $\mathcal{C}_{\mathrm{pde}}$ receive no direct interior contribution. Their amplitudes must therefore be selected by the boundary and initial-condition terms together with the empirical collocation constraints.

The initialization scale $\sigma_{\bsw}$ influences which portions of these weakly constrained spectral regions are effectively represented by the network. For the transport equation, where $\mathcal{C}_{\mathrm{pde}}$ is the characteristic line $\xi_\tau+c\xi_x=0$, changing $\sigma_{\bsw}$ modifies the spectral coverage around and along this direction. This may expose the network to high-frequency modes that incur little interior PDE penalty. Such modes can enrich the approximation space but may also yield low-loss solutions with poor predictive accuracy if they are not sufficiently constrained by the remaining terms.

This behavior is observed empirically in the multiscale transport problem: Gaussian initialization with $\sigma_{\bsw}\in[8,12]$ frequently yields solutions with small training loss but large prediction error, whereas $\sigma_{\bsw}\in[2,6]$ produces substantially more reliable approximations. A similar sensitivity is also observed for the PDE-informed initialization (Figure~\ref{fig:Failure}).

This observation also motivates the search ranges used for $\sigma_{\bsw}$ in our hyperparameter optimization. The bounds are selected empirically to avoid regimes that repeatedly exhibit low-training-loss but inaccurate solutions, rather than being arbitrary numerical choices.

\begin{figure}[htb]
    \centering
    \includegraphics[width=0.70\textwidth]{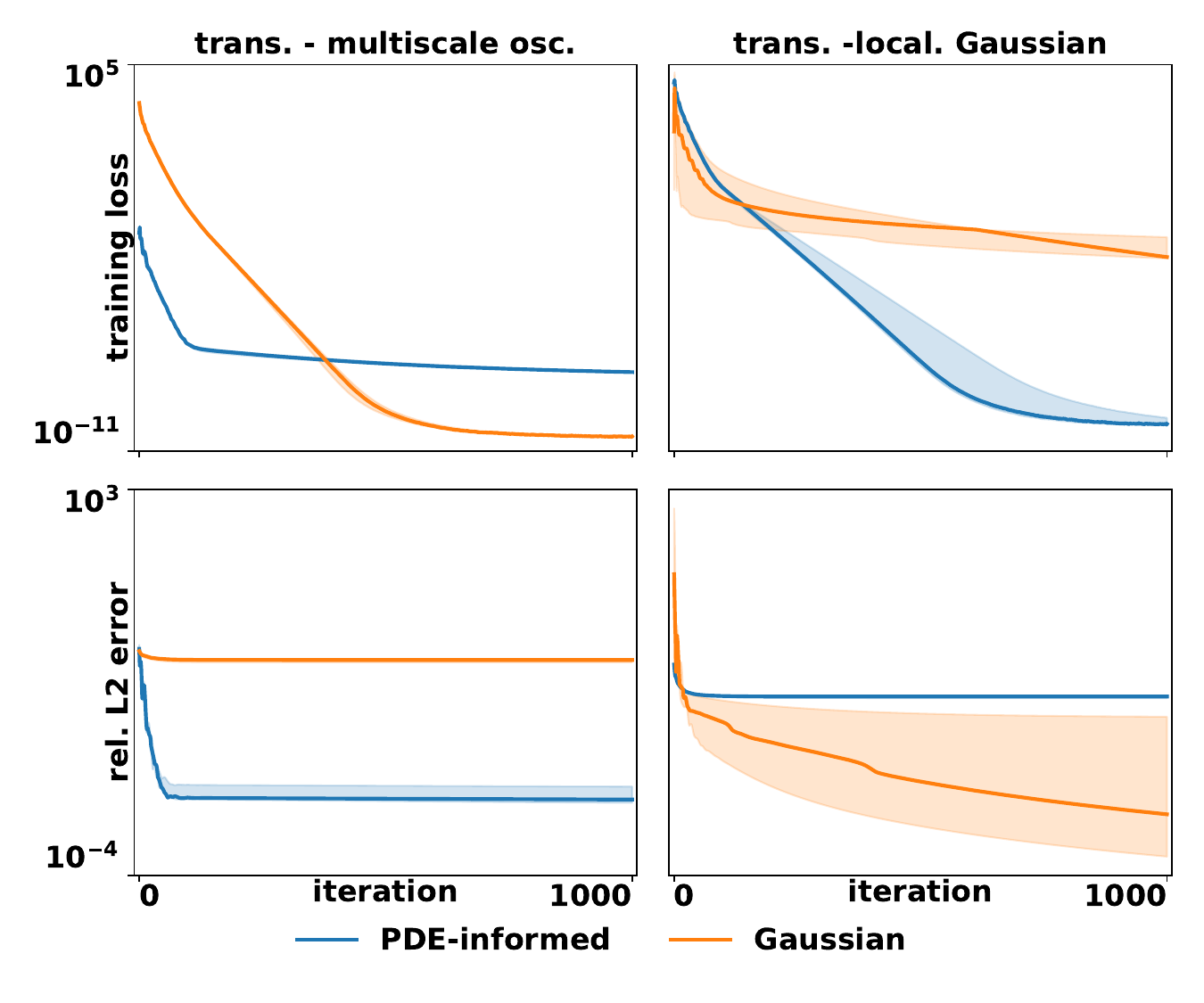}
    \caption{Examples illustrating the decoupling between empirical training loss and prediction accuracy. The first row reports the training loss, while the second row shows the relative $L^2$ prediction error. In the first column, the Gaussian initialization achieves lower training loss but a larger prediction error. In the second column, the PDE-informed initialization achieves lower training loss, whereas the Gaussian initialization yields the more accurate solution.}
    \label{fig:Failure}
\end{figure}

\subsection{PDE-Informed vs.\ Standard Initialization}

In this section, we compare the proposed PDE-informed initialization with the standard Gaussian initialization through a qualitative analysis of the training dynamics. In particular, we investigate whether the proposed initialization leads to: (i) lower final training loss, (ii) lower prediction error, (iii) more stable dynamics across different random seeds, and (iv) faster convergence toward a stable minimum. We also examine the evolution of different spectral regions during training to assess whether the proposed initialization reduces frequency-dependent disparities in the learning dynamics.

Our experiments provide empirical evidence supporting the following observations:
\begin{itemize}
    \item The theoretical framework developed in the previous sections provides meaningful insight into the optimization dynamics of PINNs.
    \item The proposed PDE-informed initialization mitigates frequency bias, often resulting in faster convergence, improved predictive accuracy, and more stable optimization dynamics.
    \item The proposed initialization generally exhibits more stable behavior across random seeds than the standard Gaussian initialization.
    \item The observed improvements extend beyond the theoretical setting to semilinear PDEs, multilayer neural networks, and ADAM optimization. Moreover, the proposed initialization can improve the performance of established PINN frameworks such as JAX-PI without introducing additional computational cost.
\end{itemize}

We begin by comparing both initialization strategies within the theoretical setting described in the previous sections. To remain consistent with the assumptions underlying the residual dynamics analysis, these experiments are performed using gradient descent (GD), two-layer neural networks with a large hidden width ($m=2000$), and linear PDEs. We consider the Laplace, Helmholtz, transport, heat, and wave equations, with multiple exact solutions used across the experiments.

For both initialization strategies, the initialization scale $\sigma_{\bsw}$ is selected using \texttt{Optuna} \citep{akiba2019optuna}. For each experiment, we first identify a search interval that consistently produces physically meaningful solutions. Within this interval, \texttt{Optuna} performs 100 trials for each random seed to estimate the optimal value of $\sigma_{\bsw}$. The resulting search ranges and selected hyperparameters are reported in Table~\ref{tab:optuna_hyperparameters_sgd}.

The analysis \ref{sec:DesignPropose} suggests that the effectiveness of the proposed initialization can be interpreted through the dimension and geometry of the zero set $\mathcal{C}_{\mathrm{pde}}$, as well as the anisotropy with which $p_{\mathrm{pde}}$ grows away from this set. These properties determine the shape of the PDE-informed distribution and, consequently, the extent to which it differs from an isotropic Gaussian initialization.

The considered linear equations exhibit substantially different geometries of $\mathcal{C}_{\mathrm{pde}}$. For the Laplace equation, $\mathcal{C}_{\mathrm{pde}}$ consists only of the origin and the level sets of $p_{\mathrm{pde}}$ are isotropic. The heat equation also has an isolated zero at the origin, but its low-level sets are anisotropic due to the different scaling of the temporal and spatial frequencies. For the transport equation, $\mathcal{C}_{\mathrm{pde}}$ is a characteristic line, while for the wave equation it consists of characteristic branches. Finally, for the Helmholtz equation, the set $\mathcal{C}_{\mathrm{pde}}$ forms a circle in frequency space.

Figure~\ref{fig:SGD2layer} presents the gradient descent training dynamics under both initialization strategies. For the Laplace equation, the two initializations exhibit comparable dynamics. Moderate but consistent improvements are observed for the heat equation, particularly in prediction accuracy. More pronounced gains are obtained for the Helmholtz, transport, and wave equations, where the PDE-informed initialization achieves faster convergence and lower prediction errors.

\begin{figure}[p]
    \centering
    \includegraphics[width=\textwidth]{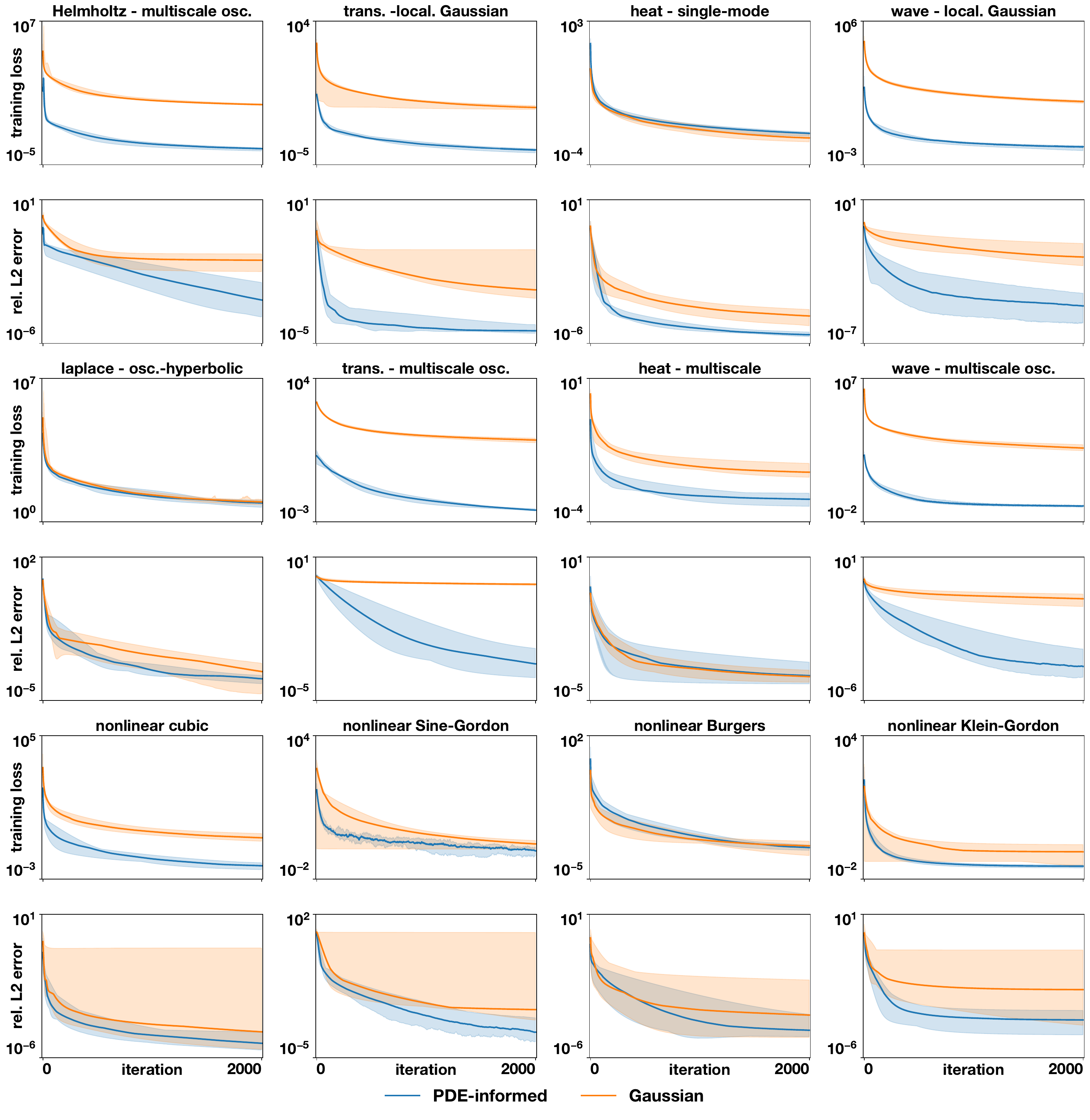}
    \caption{GD training dynamics for linear and nonlinear PDEs under the standard Gaussian and the proposed PDE-informed initialization. The odd rows report the training loss, while the even rows show the relative $L^2$ prediction error. The first two pairs of rows correspond to linear PDEs. The first pair includes the Helmholtz, transport, heat, and wave equations, while the second pair includes the Laplace, transport, heat, and wave equations. For the transport, heat, and wave equations, the two pairs use different exact solutions, thereby assessing the initialization across distinct spectral contents for the same governing PDE. The last pair of rows presents four nonlinear examples: a manufactured nonlinear Laplace equation, the Burgers equation, the Sine--Gordon equation, and the nonlinear Klein--Gordon equation. For the Laplace equation, both initialization strategies exhibit similar dynamics. In contrast, the proposed PDE-informed initialization yields moderate improvements for the heat equation and more pronounced gains for the Helmholtz, transport, and wave equations. For the nonlinear problems, the proposed initialization generally improves prediction accuracy and stability across random seeds.}\label{fig:SGD2layer}
\end{figure}

\begin{figure}[tb]
    \centering
    \includegraphics[width=\textwidth]{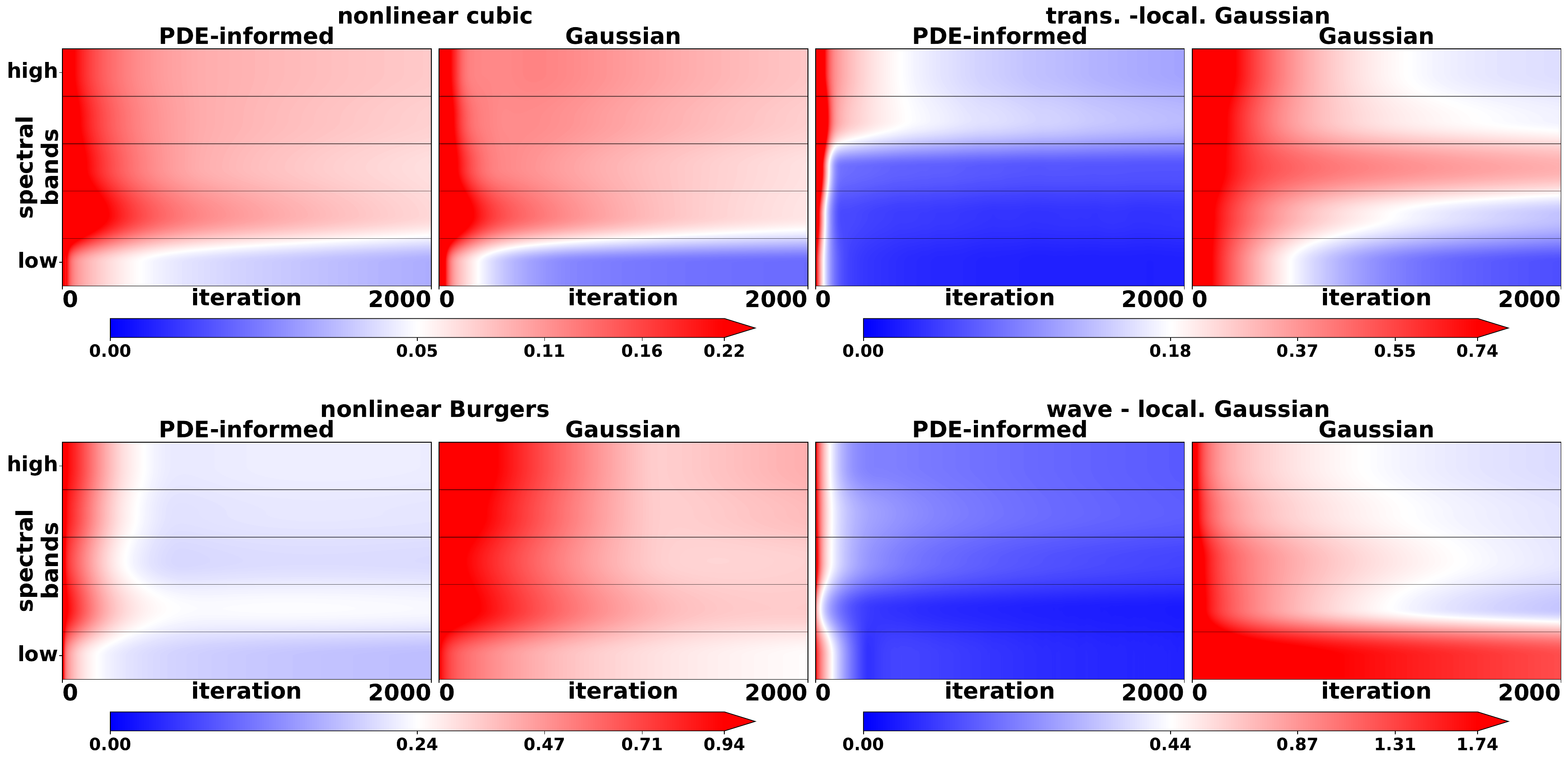}
    \caption{Evolution of the relative ring-wise spectral error during training for four representative PDEs under the proposed PDE-informed and standard Gaussian initializations. At each iteration, the prediction error is transformed into the Fourier domain, and its relative energy is computed over five concentric radial-frequency bands with logarithmically spaced radii, excluding the zero-frequency mode. The vertical axis orders these bands by increasing radial frequency, from the innermost low-frequency band to the outermost high-frequency band, while the horizontal axis denotes the training iteration. Thus, the vertical coordinate represents an ordered frequency band rather than the actual frequency magnitude or ring radius. For each PDE, the PDE-informed and Gaussian heatmaps share the same color scale, allowing a direct comparison of their spectral learning dynamics. The PDE-informed initialization generally produces a more uniform reduction of the spectral error across frequency bands, with larger differences for problems whose low-symbol regions exhibit stronger anisotropy or nontrivial characteristic geometries.}
    \label{fig:FBmitigation}
\end{figure}

\begin{table*}[htb]
\centering
\small
\caption{Learning rates, selected initialization scales, and Optuna search intervals used in the experiments. The examples correspond to Figure~\ref{fig:ADAM2layer}.}
\label{tab:optuna_hyperparameters_adam}
\begin{tabular}{|c|c|c|c|c|c|c|}
\hline
example &
$\eta_{\mathrm{GI}}$ &
$\sigma_{\bsw}^{\mathrm{GI}}$ &
GI range &
$\eta_{\mathrm{PI}}$ &
$\sigma_{\bsw}^{\mathrm{PI}}$ &
PI range \\
\hline

Helmholtz Multiscale Osc.
& $5\cdot10^{-3}$
& $0.97$
& $[0.3,3]$ 
& $1\cdot10^{-3}$
& $0.71$
& $[0.01,1]$\\
\hline

Trans. Local. Gaussian
& $5\cdot10^{-4}$
& $2.1$
& $[0.1,10]$
& $5\cdot10^{-4}$
& $0.096$
& $[10^{-4},1]$ \\
\hline

Heat Single-Mode
& $10^{-4}$
& $0.27$
& $[0.1,10]$
& $10^{-4}$
& $0.086$
& $[10^{-5},0.1]$ \\
\hline

Wave Multiscale Osc.
& $10^{-3}$
& $2.6$
& $[0.5,10]$
& $7\cdot10^{-4}$
& $53.5$
& $[0.01,100]$ \\
\hline

Nonlinear Cubic
& $7\cdot10^{-5}$
& $0.095$
& $[0.01,10]$
& $10^{-4}$
& $0.03$
& $[0.005,0.5]$ \\
\hline

Nonlinear Burgers
& $10^{-4}$
& $0.11$
& $[0.05,0.5]$
& $10^{-4}$
& $7.78$
& $[1,10]$ \\
\hline

Nonlinear Sine--Gordon
& $5\cdot10^{-4}$
& $0.27$
& $[0.1,1]$
& $5\cdot10^{-4}$
& $0.074$
& $[0.01,0.1]$ \\
\hline

Nonlinear Klein--Gordon
& $2\cdot10^{-3}$
& $0.29$
& $[0.01,10]$
& $7\cdot10^{-3}$
& $0.0037$
& $[10^{-4},1]$ \\
\hline

\end{tabular}
\end{table*}

\begin{figure}[tb]
    \centering
    \includegraphics[width=\textwidth]{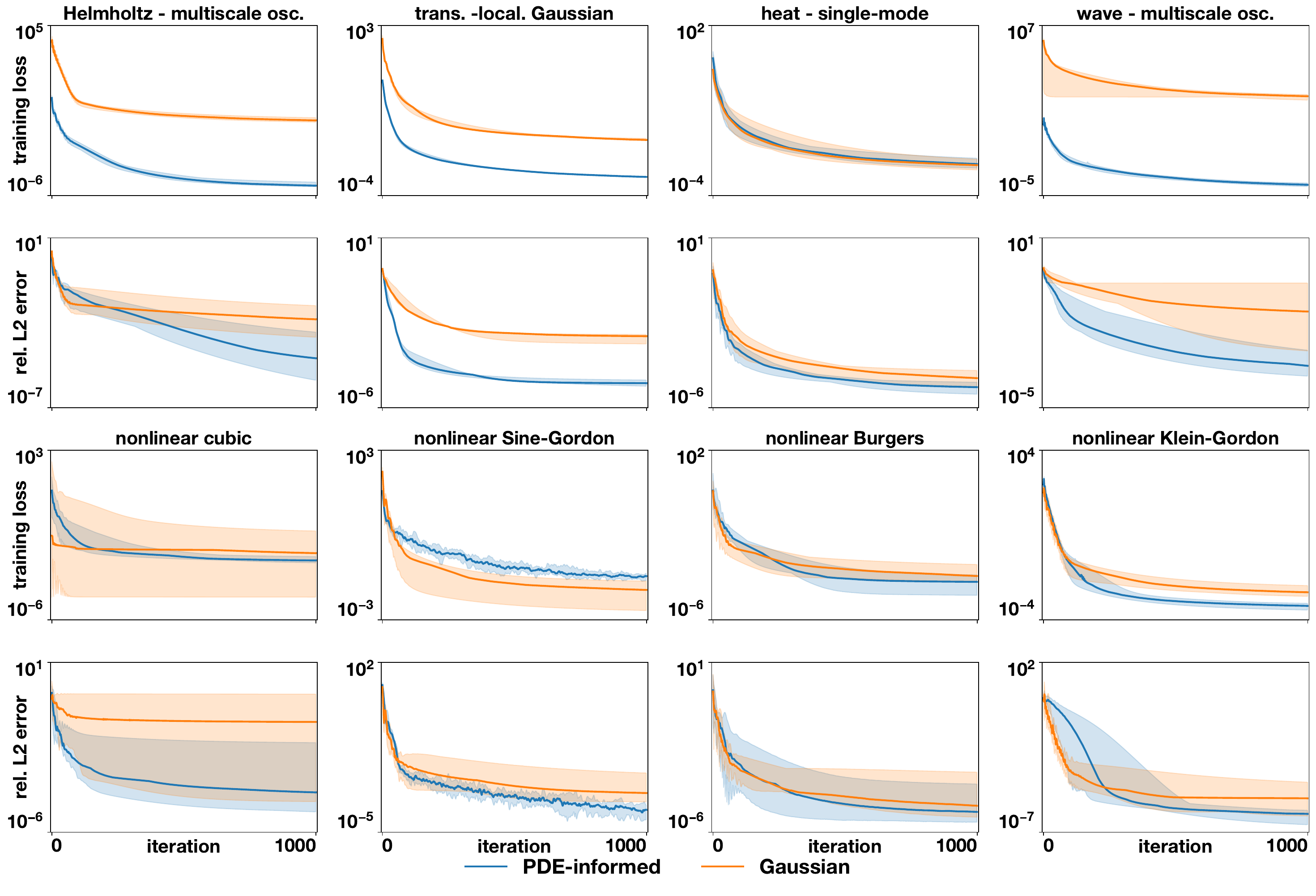}
    \caption{ADAM training dynamics for linear and nonlinear PDEs under the standard Gaussian and the proposed PDE-informed initialization. The odd rows report the training loss, while the even rows show the relative $L^2$ prediction error. The linear experiments include the Helmholtz, transport, heat, and wave equations, while the remaining rows present four nonlinear PDEs: a manufactured nonlinear Laplace equation, the Burgers equation, the Sine--Gordon equation, and the nonlinear Klein--Gordon equation. Consistent with the GD experiments, the proposed initialization improves the optimization dynamics and prediction accuracy for the linear problems, with moderate gains for the heat equation and more pronounced improvements for the Helmholtz, transport, and wave equations. Similar improvements are observed for most of the nonlinear problems, with generally faster convergence, lower prediction errors, and more stable behavior across random seeds.}
    \label{fig:ADAM2layer}
\end{figure}

These results are consistent with the shape of $\mathcal{C}_{\mathrm{pde}}$. When the zero set is an isolated point and the surrounding level sets are isotropic, as in the Laplace equation, the PDE-informed distribution remains qualitatively similar to radial Gaussian sampling. Introducing anisotropy, as in the heat equation, produces a more structured initialization and leads to a clearer improvement. The largest differences arise when $\mathcal{C}_{\mathrm{pde}}$ itself has a nontrivial geometry, such as the characteristic lines or branches of the transport and wave equations or the circular zero set of the symbol of Helmholtz equation. Overall, the observed gains are closely related to the mismatch between the geometry $\mathcal{C}_{\mathrm{pde}}$ induced by $p_{\mathrm{pde}}$ and the geometry of isotropic Gaussian sampling.

Following the analysis of the optimization dynamics, we next investigate whether the proposed initialization mitigates the frequency bias during training. To this end, we monitor the evolution of the prediction error in the Fourier domain throughout training. At selected training iterations, we compute the two-dimensional discrete Fourier transform of the prediction error and evaluate its relative energy over concentric frequency rings. Specifically, letting $R_k$ denote the $k$-th radial ring in the frequency domain, excluding the zero-frequency mode, we define
\begin{equation}
    E_k = \frac{\|\widehat{f}_{\theta}-\widehat{u}\|_{R_k}}{\|\widehat{u}\|_{R_k}},
\end{equation}
where $\|\cdot\|_{R_k}$ denotes the $L^2(R_k)$ Euclidean norm restricted to the frequencies contained in $R_k$. The Fourier domain is partitioned into five concentric rings with logarithmically spaced radii on a $101\times101$ grid, with the first ring beginning immediately after the zero-frequency component.

Figure~\ref{fig:FBmitigation} shows the evolution of $E_k$ during training for four representative PDEs. The two initializations exhibit similar behavior for the nonlinear cubic equation, whereas the PDE-informed initialization produces a more uniform reduction of the spectral error across frequency bands for Burgers, transport, and wave, reaching low error levels earlier than the standard Gaussian initialization. These results provide empirical evidence that the proposed initialization mitigates the PDE-induced frequency bias.

\subsection{PDE-Informed vs.\ Gaussian initialization beyond the theoretical setting}

Having validated the proposed initialization within the theoretical regime, we now investigate whether the observed behavior persists in more realistic training settings. In particular, we consider nonlinear PDEs, ADAM optimization, and multilayer neural network architectures.

\subsubsection{Nonlinear dynamics}

We next investigate whether the previous results extend to nonlinear equations. In particular, we consider semilinear PDEs for which $p_{\mathrm{pde}}$ is constructed from the symbol associated with the principal linear differential operator. Although this symbol does not fully characterize the frequency dynamics of a nonlinear equation, it provides the geometric information used to construct the PDE-informed initialization.

Figure~\ref{fig:SGD2layer} presents the results for a manufactured nonlinear Laplace equation, the Burgers equation with $\nu=1$, the Sine--Gordon equation \citep{ablowitz1997numerical}, and the nonlinear Klein--Gordon equation \citep{alsisi2024analytical}. These examples inherit different low-symbol geometries from their principal linear operators. In particular, the nonlinear Laplace equation has a Laplace-like symbol, the Burgers equation has a heat-like symbol, while the Sine--Gordon and Klein--Gordon equations have wave-like symbols.

The observed behavior reflects these differences. For the nonlinear Laplace and Burgers equations, the two initializations exhibit relatively similar training-loss dynamics, while the PDE-informed initialization produces a clearer improvement in prediction accuracy. In contrast, for the Sine--Gordon and Klein--Gordon equations, whose principal linear operators induce more structured wave-like low-symbol geometries, the improvement is more pronounced and is observed in both the training loss and prediction error. In general, the proposed density $\rho^{\mathrm{PI}}_{\bsw}$ improves predictive accuracy across the nonlinear problems considered.

\begin{figure}[ht]
    \centering
    \includegraphics[width=\textwidth]{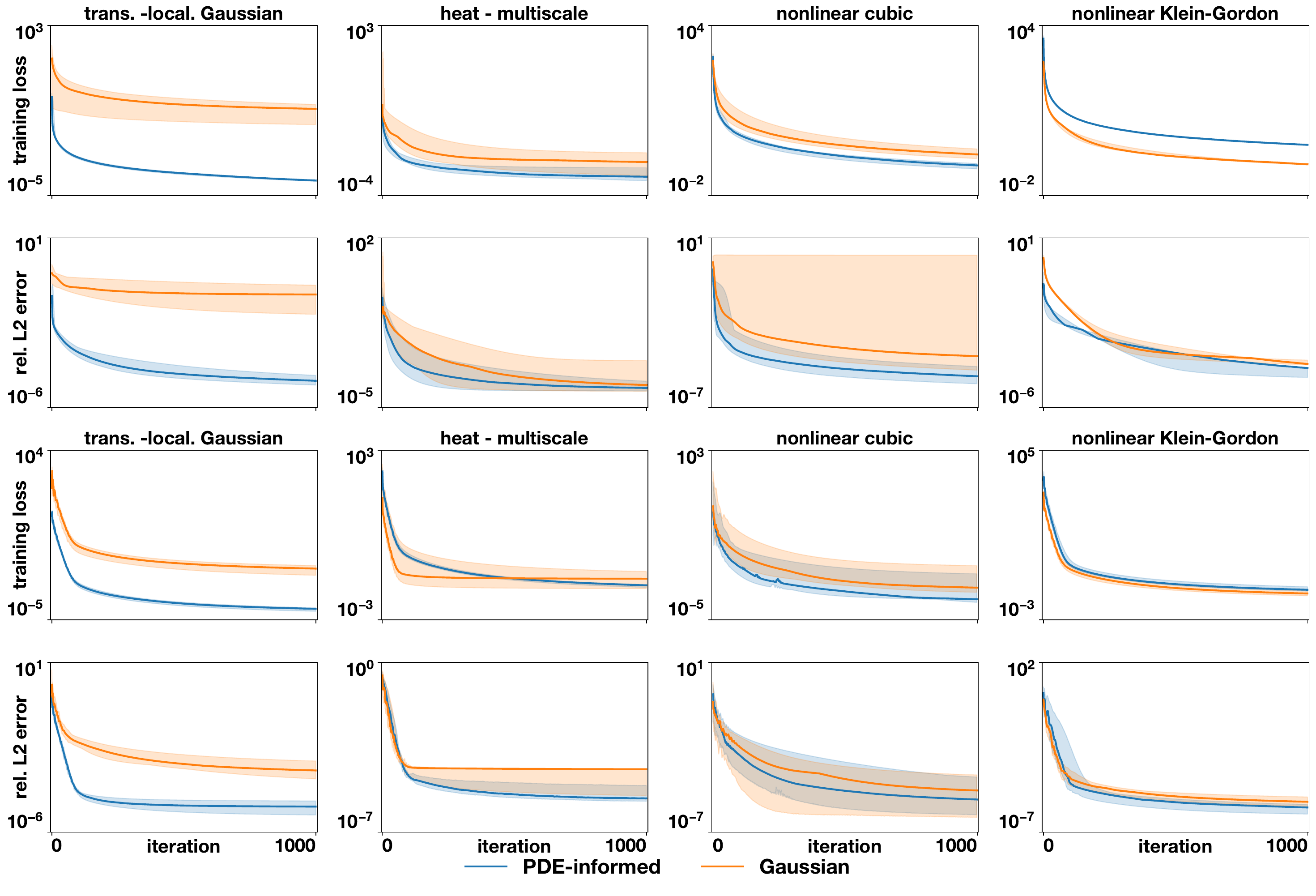}
    \caption{Multilayer experiments under GD and ADAM optimization. The first two rows show the gradient descent training dynamics, while the third and fourth rows correspond to ADAM. The odd rows report the training loss, whereas the even rows show the relative $L^2$ prediction error. Both optimizers are evaluated on the transport, heat, nonlinear Laplace (cubic), and nonlinear Klein--Gordon equations. Consistent with the two-layer experiments, the proposed PDE-informed initialization improves both the optimization dynamics and predictive accuracy under both optimization methods. The improvements are particularly pronounced for the transport equation, while more moderate but still noticeable gains are observed for the heat equation. Similar improvements are also observed for the two nonlinear problems, demonstrating that the benefits of the proposed initialization extend to deeper neural network architectures and adaptive optimization methods.}
    \label{fig:multilayer}
\end{figure}

\subsubsection{ADAM dynamics}

In this section, we investigate whether the improvements observed under GD remain when using the ADAM optimizer. In addition, we employ an exponential learning-rate decay schedule with decay factor $0.90$, updated every $50$ iterations. Although the theoretical analysis is derived in the continuous-time gradient-flow regime associated with GD, these experiments assess whether the proposed initialization retains its benefits under a widely used adaptive optimization method.

Figure~\ref{fig:ADAM2layer} presents the ADAM training dynamics for representative linear and nonlinear PDE examples. The results are consistent with those obtained using GD. For the linear equations, moderate improvements are observed for the heat equation, while the proposed initialization produces more pronounced gains for the Helmholtz, transport, and wave equations. Similarly, for the nonlinear examples, the proposed initialization generally achieves faster convergence, lower prediction errors, and more stable behavior across random seeds. Although the standard Gaussian initialization may occasionally achieve lower training losses for individual seeds, these improvements do not necessarily translate into better predictions.

\subsubsection{Multilayer case}

Finally, we investigate whether the previous conclusions obtained for two-layer networks extend to deeper architectures. In particular, we consider the linear transport and heat equations, as well as the nonlinear manufactured Laplace and Klein--Gordon equations, using a Fourier Feature network with architecture $[2,2000,200,1]$. The additional layers are initialized using the standard Glorot initialization, while the hidden layers use the $\tanh$ activation function. As in the previous experiments, we use \texttt{Optuna} with 50 trials to estimate the optimal value of $\sigma_{\bsw}$.

Figure~\ref{fig:multilayer} shows the training dynamics for both GD and ADAM optimization in the multilayer setting. The results remain consistent with the previous two-layer experiments: the proposed PDE-informed initialization improves both convergence and prediction accuracy for all considered equations and optimization methods. The relative improvements for the transport and heat equations also remain qualitatively consistent with the two-layer results, with larger gains for the transport equation and more moderate gains for the heat equation. The GD and ADAM dynamics exhibit similar qualitative behavior, suggesting that the benefits of the proposed initialization are not restricted to the theoretical two-layer setting or to a specific optimization algorithm.

\subsection{Validation on State-of-the-Art PINN Benchmarks}

\begin{figure}[tb]
    \centering
    \includegraphics[width=\textwidth]{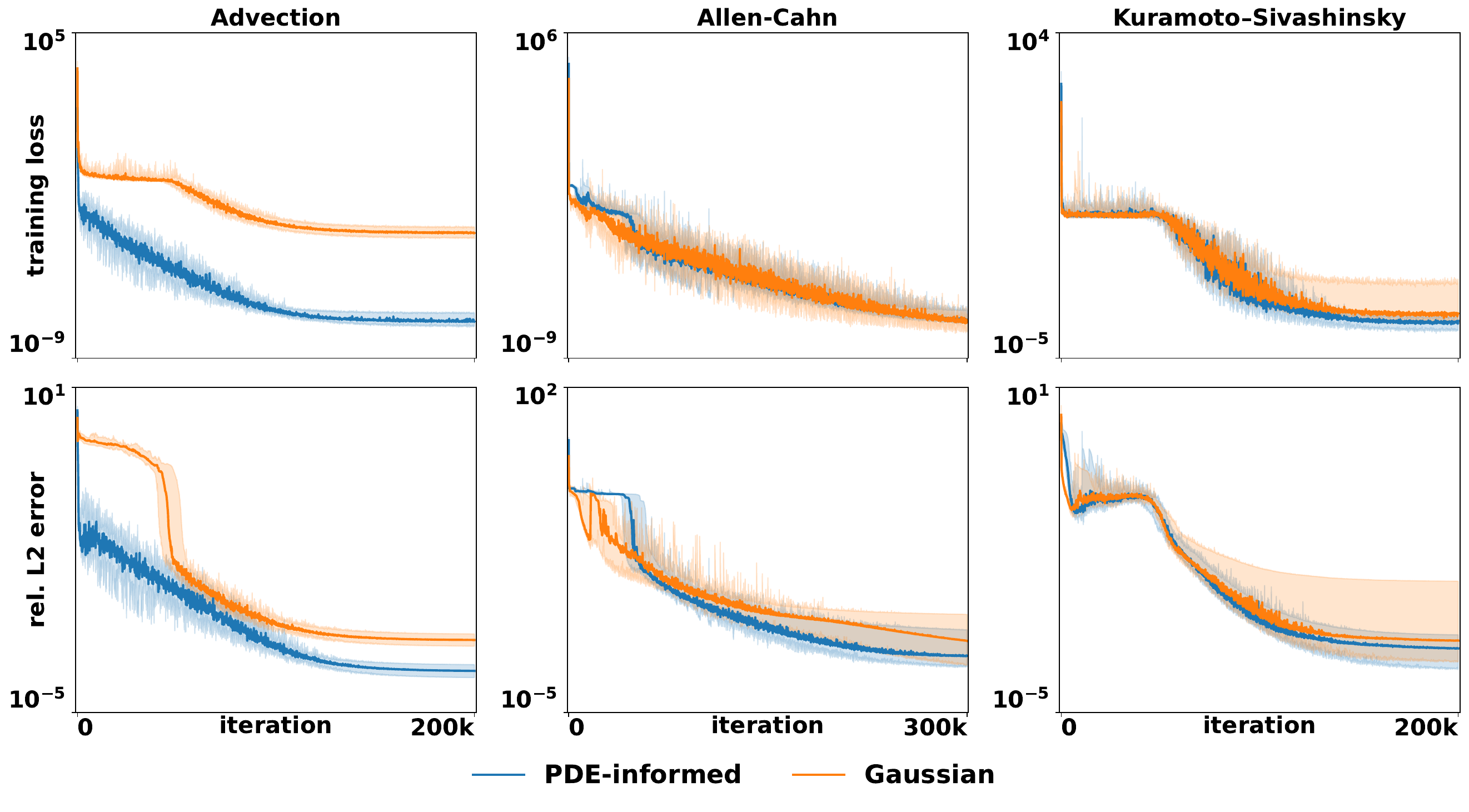}
    \caption{Training dynamics on the JAX-PI benchmarks using the recommended SOTA configuration. The first row shows the training loss, while the second row reports the relative $L^2$ prediction error. The columns correspond to the advection, Allen--Cahn, and Kuramoto--Sivashinsky equations, respectively. The blue curves denote the proposed PDE-informed initialization (using \eqref{eq:rhoPI_linear} for the linear problem and \eqref{eq:rhoPI_general} with $\lambda=0.5$ for the nonlinear problems), while the orange curves correspond to the standard Gaussian initialization. Solid lines represent the median over ten random seeds and the shaded regions indicate the range between the best and worst runs.}\label{fig:JAXPI-SOTA}
\end{figure}

\begin{figure}[h]
    \centering
    \includegraphics[width=\textwidth]{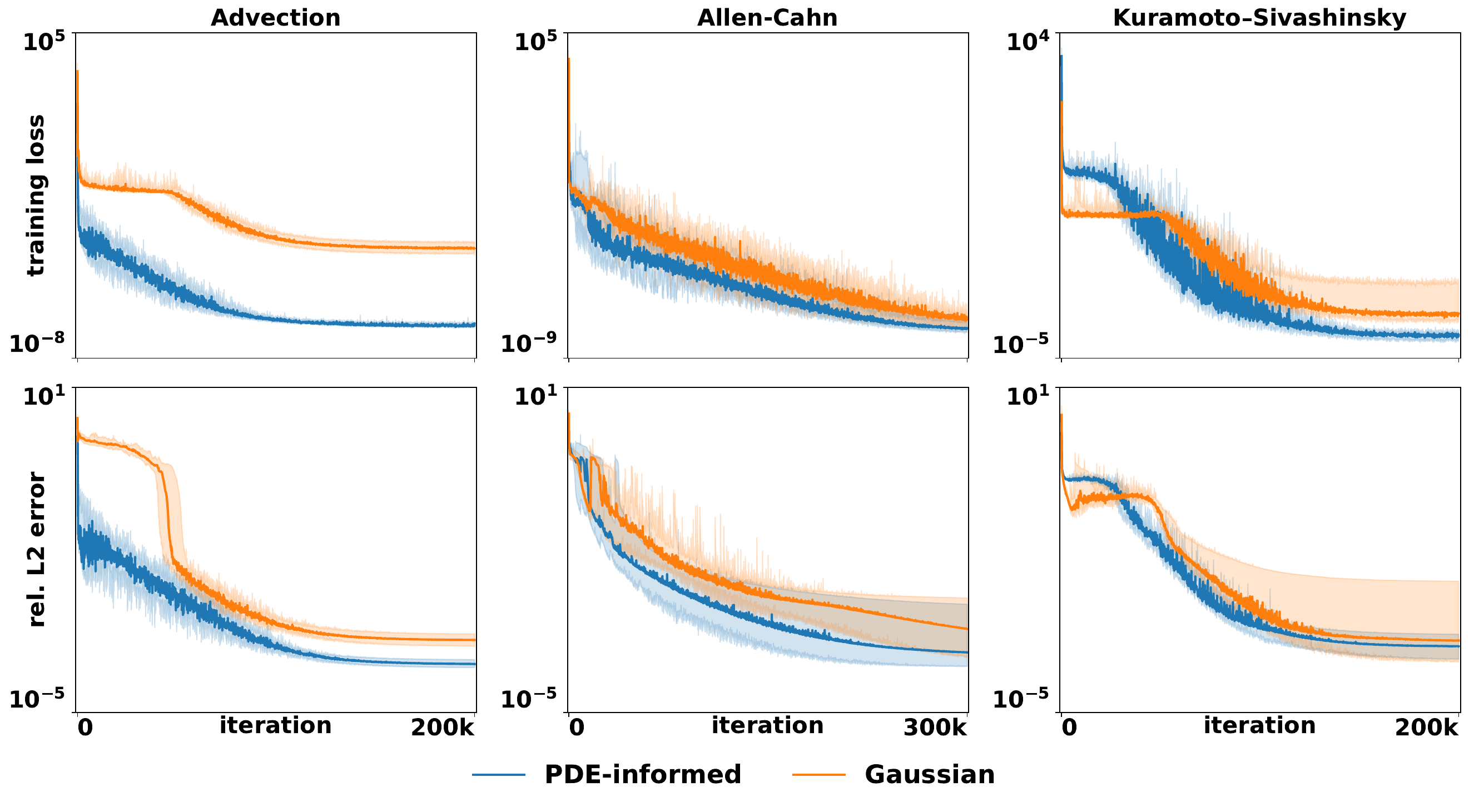}
    \caption{Training dynamics obtained with the proposed PDE-informed initialization under the simplified baseline configuration. The layout is identical to Figure~\ref{fig:JAXPI-SOTA}. The results demonstrate that the baseline configuration achieves optimization dynamics comparable to those of the full SOTA pipeline, with the Allen--Cahn equation exhibiting noticeably faster convergence during the early stages of training.}
    \label{fig:JAXPI-baseline}
\end{figure}

This section evaluates the proposed PDE-informed initialization in more realistic and challenging training scenarios using the JAX-PI repository \citep{wang2023expert}. Unlike the controlled experiments considered in the previous sections, JAX-PI provides a collection of benchmark PDE problems together 4 or 5 multilayer networks with width of 256, frozen FF embeddings, and long training schedules and carefully tuned optimization strategies that better reflect the practical challenges encountered in modern PINNs. Our objective is to determine whether the proposed initialization remains beneficial when integrated into a state-of-the-art PINN framework without modifying either the network architecture or the optimization pipeline.

The JAX-PI repository has become a widely used benchmark for evaluating PINN methodologies due to its standardized implementations, extensive ablation studies, and carefully optimized training procedures. In particular, it incorporates several stabilization techniques designed to improve training in challenging PDE problems, including causality training, adaptive loss weighting, and random weight factorization. Furthermore, the repository provides strong baseline configurations for each benchmark problem, making it a suitable environment for assessing whether the proposed initialization can improve existing state-of-the-art training settings.

To validate our findings, we tested the proposed initialization on the advection (\texttt{adv}), Allen--Cahn (\texttt{allen\_cahn}) and Kuramoto--Sivashinsky (\texttt{ks\_chaotic}) equations available in the repository. The only modification with respect to the original implementations is the initialization strategy, where the parameter $\rho^{\mathrm{PI}}_{\bsw}$ was tuned through a grid search over powers of two. For the \texttt{adv} example, Dirichlet boundary conditions were used instead of imposing boundary conditions technique, since the proposed initialization is not compatible with the periodic setting. For linear PDEs we use the PDE-informed initialization \eqref{eq:rhoPI_linear} ($\lambda=1$), whereas for nonlinear PDEs we employ the generalized PDE-informed initialization \eqref{eq:rhoPI_general} with $\lambda=0.5$.

In addition, we investigated whether the proposed initialization remains effective under the different training configurations recommended by JAX-PI. Specifically, we considered two variants. The first corresponds to the recommended state-of-the-art (SOTA) configuration, which employs all stabilization techniques suggested by the repository. The second is a simplified baseline configuration that only includes the most classical and widely used components: NTK/gradient-based loss weighting, learning-rate decay, and non-dimensionalization. We refer to this latter setting as the \emph{baseline} configuration. Both variants were compared against the original SOTA configuration using the standard Gaussian initialization.

Figure~\ref{fig:JAXPI-SOTA} compares the proposed PDE-informed initialization with the standard Gaussian initialization using the original SOTA configuration recommended by JAX-PI. Across all three benchmarks, the proposed initialization produces either faster or comparable optimization dynamics while consistently reducing the variability across random seeds. The improvement is particularly pronounced for the advection equation, where both the training loss and the relative $L^2$ error decrease substantially faster than with the Gaussian initialization. For the nonlinear problems, the differences are smaller but remain consistent throughout training, leading to improved or comparable final prediction errors.

Figure~\ref{fig:JAXPI-baseline} evaluates the proposed initialization under the simplified baseline configuration. Interestingly, the dynamics obtained with the baseline and SOTA configurations are better or similar for the PDE-informed initialization. This observation suggests that part of the improvement provided by the additional stabilization techniques incorporated in JAX-PI may already be achieved through an appropriate initialization of the Fourier frequencies. In particular, for the Allen--Cahn and Kuramoto--Sivashinsky equations, the baseline configuration exhibits noticeably faster convergence in both the training loss and the relative $L^2$ error, and generally yields a moderately better prediction. 

A possible explanation for this behavior is that some of the stabilization mechanisms employed by JAX-PI, such as Random Weight Factorization and Respecting Temporal Causality, also act to alleviate frequency bias during training, although through optimization heuristics rather than by explicitly modifying the frequency distribution at initialization. Consequently, these techniques and the proposed PDE-informed initialization appear to address similar optimization difficulties from different perspectives. An important distinction, however, is that the proposed initialization introduces virtually no additional computational cost, since it only modifies the sampling distribution of the Fourier frequencies before training begins.

Finally, Figure~\ref{fig:JAXPI-predictions} presents a qualitative comparison between the ground-truth solutions and the corresponding predictions obtained with the PDE-informed and Gaussian initializations using the SOTA configuration. Both methods accurately reconstruct the target solutions, demonstrating that the overall approximation quality remains high in all benchmarks. Nevertheless, the absolute error maps reveal a clear advantage for the proposed initialization, particularly for the advection equation, where the prediction error is substantially reduced over the entire computational domain. The Allen--Cahn and Kuramoto--Sivashinsky equations exhibit smaller visual differences, consistent with the quantitative error curves, while still showing slightly lower reconstruction errors for the PDE-informed initialization.

\begin{figure}[t]
    \centering
    \includegraphics[width=\textwidth]{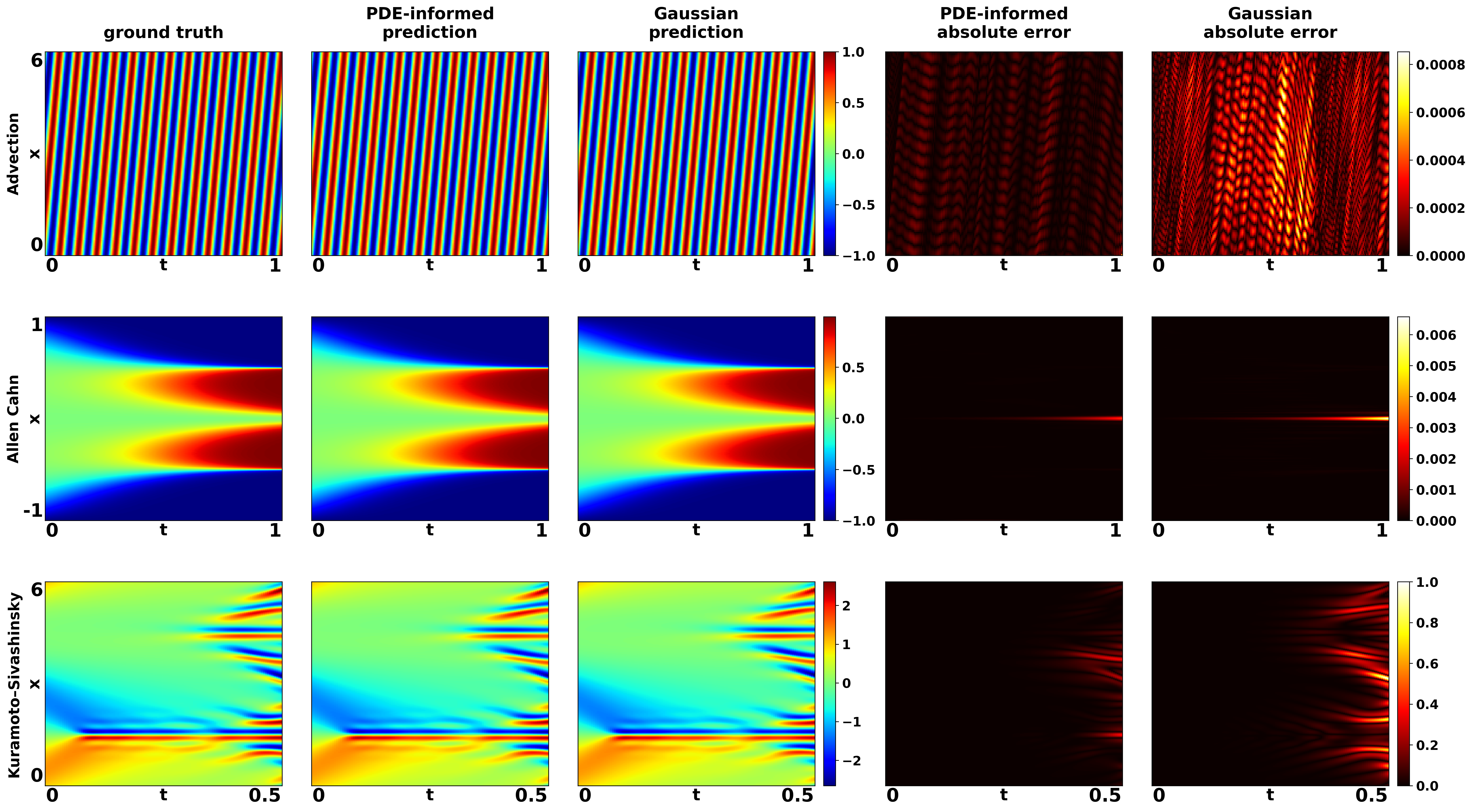}
    \caption{Qualitative comparison of the reconstructed solutions obtained with the SOTA configuration for Gaussian initialization and baseline configuration for PDE-Informed initialization. Each row corresponds to one benchmark equation (advection, Allen--Cahn, and Kuramoto--Sivashinsky). The columns show the ground-truth solution, the prediction obtained with the proposed PDE-informed initialization, the prediction obtained with the standard Gaussian initialization, the absolute error of the PDE-informed initialization, and the absolute error of the Gaussian initialization. The predictions produced by both methods closely match the ground truth, while the proposed initialization generally yields lower reconstruction errors, particularly for the advection equation.}
    \label{fig:JAXPI-predictions}
\end{figure}

\section{Discussion}

In this work, we developed a theoretical framework describing the residual frequency dynamics of PINNs with Fourier-feature neural networks in the NTK regime. Under the assumption of frozen Fourier frequencies, the dynamics are governed by equation \eqref{eq:FreezeDynamics}, which explicitly reveals the interaction between the initialization density ($\rho_{\bsw}$) and the PDE symbol ($p_{\mathrm{pde}}$). The analysis shows that the convergence rate of the residual is strongly influenced by regions of frequency space where the PDE symbol becomes small or vanishes. In particular, the dimension and geometry of the zero set $\mathcal C_{\mathrm{pde}}=\{\bxi:p_{\mathrm{pde}}(\bxi)=0\}$ determine the spectral directions along which the interior PDE residual provides the weakest control. As a consequence, the spectral properties of both the differential operator and the initialization distribution play a fundamental role in determining the optimization dynamics of PINNs.

Motivated by this observation, we proposed a PDE-informed initialization strategy that adapts the sampling distribution of Fourier frequencies to the geometry induced by the PDE symbol. The proposed initialization increases the concentration of frequencies near regions where $p_{\mathrm{pde}}$ is small, thereby compensating for the uneven convergence rates predicted by the theory. Unlike standard Gaussian initialization, which samples frequencies isotropically, the proposed approach explicitly incorporates information about the underlying differential operator into the initialization process.

The theoretical predictions were validated through extensive numerical experiments involving PDEs with markedly different geometries of $\mathcal C_{\mathrm{pde}}$. Across a broad range of examples, the proposed initialization consistently mitigated frequency bias, accelerated optimization, improved residual convergence, and reduced variability across random seeds. The magnitude of the improvement was closely associated with the dimension and shape of $\mathcal C_{\mathrm{pde}}$, as well as with the geometry of the surrounding regions where $p_{\mathrm{pde}}$ remains small. For the Laplace equation, $\mathcal C_{\mathrm{pde}}$ reduces to the origin and the two initialization strategies exhibit comparable dynamics. The heat equation also has an isolated zero, but its low-symbol region is anisotropic, leading to moderate improvements. In contrast, substantially larger gains are observed for the transport, wave, and Helmholtz equations, where $\mathcal C_{\mathrm{pde}}$ contains extended structures in frequency space. These results suggest that the advantage of the PDE-informed initialization becomes more pronounced when the weakly constrained spectral region has a nontrivial geometry that is not efficiently captured by an isotropic Gaussian initialization.

Furthermore, the empirical results indicate that the proposed methodology extends well beyond the theoretical assumptions used in the analysis. In particular, the initialization remains effective for nonlinear PDEs, multilayer neural network architectures, and ADAM optimization. We also demonstrated that the proposed initialization can improve the performance of established PINN frameworks, such as JAX-PI, without incurring additional computational cost, since the method only modifies the sampling distribution used during initialization.

Compared with previous approaches to mitigating frequency bias in PINNs, the proposed methodology differs in both its motivation and its design. Existing initialization strategies are generally PDE-agnostic, aiming to improve optimization through generic weight initialization or transfer learning without explicitly exploiting the structure of the governing equation \citep{liu2022novel,tarbiyati2024weight}. Other frequency-aware methods instead modify the network architecture, progressively introduce higher frequencies during training, or transform the learning problem to facilitate the optimization of oscillatory solutions \citep{cai2020phase,cai2024batch,LI2023114963,XU2025113905,LIU2025106886}. In contrast, our approach directly incorporates information from the PDE via its Fourier symbol, yielding an initialization distribution tailored to the spectral structure of the underlying differential operator. Moreover, unlike methods that alter the training procedure or network architecture, the proposed strategy is implemented solely at initialization, making it straightforward to integrate into existing PINN frameworks.

Despite these promising results, several limitations remain. The most important one is that the proposed PDE-informed initialization relies on the availability of a meaningful PDE symbol ($p_{\mathrm{pde}}$). Consequently, the current methodology is naturally applicable to linear and semilinear PDEs with constant coefficients, where the principal linear operator provides a well-defined frequency characterization of the problem. For fully nonlinear PDEs, however, there is generally no unique linear symbol that captures the dominant dynamics, making the construction of a PDE-informed frequency distribution considerably more challenging. Extending the present framework to such settings is an important direction for future research and may require developing nonlinear or solution-dependent frequency descriptors.

More broadly, the framework developed in this paper suggests a general principle for neural network initialization: optimization can be improved by aligning the initialization distribution with the spectral structure of the operator governing the learning problem. From this perspective, the methodology introduced here may extend beyond PINNs to inverse problems, integral and nonlocal equations, operator-learning architectures, and other machine-learning settings in which the optimization dynamics are shaped by the spectral properties of an underlying operator. We hope these results contribute to a deeper understanding of frequency bias in scientific machine learning and provide a foundation for developing operator-aware training strategies for future neural network models.

\section*{Acknowledgments}

We gratefully acknowledge the support of the Agencia Nacional de Investigación y Desarrollo (ANID), Chile. This research was funded by the National Center for Artificial Intelligence CENIA FB210017, Basal ANID. M Courdurier acknowledges ANID programs Fondecyt 1262107 and Fondecyt 1240200. This work was also funded by ANID – Millennium Science Initiative Program – ICN2021\_004 and FONDECYT Regular 1251394 to FSC. 

\bibliographystyle{cas-model2-names}
\bibliography{references}

\begin{thebibliography}{56}
\expandafter\ifx\csname natexlab\endcsname\relax\def\natexlab#1{#1}\fi
\providecommand{\url}[1]{\texttt{#1}}
\providecommand{\href}[2]{#2}
\providecommand{\path}[1]{#1}
\providecommand{\DOIprefix}{doi:}
\providecommand{\ArXivprefix}{arXiv:}
\providecommand{\URLprefix}{URL: }
\providecommand{\Pubmedprefix}{pmid:}
\providecommand{\doi}[1]{\href{http://dx.doi.org/#1}{\path{#1}}}
\providecommand{\Pubmed}[1]{\href{pmid:#1}{\path{#1}}}
\providecommand{\bibinfo}[2]{#2}
\ifx\xfnm\relax \def\xfnm[#1]{\unskip,\space#1}\fi
\bibitem[{Ablowitz et~al.(1997)Ablowitz, Herbst and
  Schober}]{ablowitz1997numerical}
\bibinfo{author}{Ablowitz, M.J.}, \bibinfo{author}{Herbst, B.},
  \bibinfo{author}{Schober, C.M.}, \bibinfo{year}{1997}.
\newblock \bibinfo{title}{On the numerical solution of the sine--{Gordon}
  equation}.
\newblock \bibinfo{journal}{Journal of Computational Physics}
  \bibinfo{volume}{131}, \bibinfo{pages}{354--367}.
\bibitem[{Akiba et~al.(2019)Akiba, Sano, Yanase, Ohta and
  Koyama}]{akiba2019optuna}
\bibinfo{author}{Akiba, T.}, \bibinfo{author}{Sano, S.},
  \bibinfo{author}{Yanase, T.}, \bibinfo{author}{Ohta, T.},
  \bibinfo{author}{Koyama, M.}, \bibinfo{year}{2019}.
\newblock \bibinfo{title}{Optuna: A next-generation hyperparameter optimization
  framework}, in: \bibinfo{booktitle}{Proceedings of the 25th ACM SIGKDD
  international conference on knowledge discovery \& data mining}, pp.
  \bibinfo{pages}{2623--2631}.
\bibitem[{Alsisi(2024)}]{alsisi2024analytical}
\bibinfo{author}{Alsisi, A.}, \bibinfo{year}{2024}.
\newblock \bibinfo{title}{Analytical and numerical solutions to the
  {Klein}--{Gordon} model with cubic nonlinearity}.
\newblock \bibinfo{journal}{Alexandria Engineering Journal}
  \bibinfo{volume}{99}, \bibinfo{pages}{31--37}.
\bibitem[{Aravanis et~al.(2025)Aravanis, Papadopoulos and
  Georgikos}]{aravanis2025fourier}
\bibinfo{author}{Aravanis, T.}, \bibinfo{author}{Papadopoulos, P.},
  \bibinfo{author}{Georgikos, D.}, \bibinfo{year}{2025}.
\newblock \bibinfo{title}{Fourier feature-enhanced neural networks for wind
  turbine power modeling}.
\newblock \bibinfo{journal}{Electricity} \bibinfo{volume}{6},
  \bibinfo{pages}{70}.
\bibitem[{Arora et~al.(2019a)Arora, Du, Hu, Li and Wang}]{arora2019fine}
\bibinfo{author}{Arora, S.}, \bibinfo{author}{Du, S.}, \bibinfo{author}{Hu,
  W.}, \bibinfo{author}{Li, Z.}, \bibinfo{author}{Wang, R.},
  \bibinfo{year}{2019}a.
\newblock \bibinfo{title}{Fine-grained analysis of optimization and
  generalization for overparameterized two-layer neural networks}, in:
  \bibinfo{booktitle}{International Conference on Machine Learning},
  \bibinfo{organization}{PMLR}. pp. \bibinfo{pages}{322--332}.
\bibitem[{Arora et~al.(2019b)Arora, Du, Hu, Li, Salakhutdinov and
  Wang}]{arora2019exact}
\bibinfo{author}{Arora, S.}, \bibinfo{author}{Du, S.S.}, \bibinfo{author}{Hu,
  W.}, \bibinfo{author}{Li, Z.}, \bibinfo{author}{Salakhutdinov, R.R.},
  \bibinfo{author}{Wang, R.}, \bibinfo{year}{2019}b.
\newblock \bibinfo{title}{On exact computation with an infinitely wide neural
  net}.
\newblock \bibinfo{journal}{Advances in neural information processing systems}
  \bibinfo{volume}{32}.
\bibitem[{Bhattamishra et~al.(2022)Bhattamishra, Patel, Kanade and
  Blunsom}]{bhattamishra2022simplicity}
\bibinfo{author}{Bhattamishra, S.}, \bibinfo{author}{Patel, A.},
  \bibinfo{author}{Kanade, V.}, \bibinfo{author}{Blunsom, P.},
  \bibinfo{year}{2022}.
\newblock \bibinfo{title}{Simplicity bias in transformers and their ability to
  learn sparse boolean functions}.
\newblock \bibinfo{journal}{arXiv preprint arXiv:2211.12316} .
\bibitem[{Bounnah et~al.(2025)Bounnah, Mihoubi and Larbi}]{bounnah2025physics}
\bibinfo{author}{Bounnah, Y.}, \bibinfo{author}{Mihoubi, M.K.},
  \bibinfo{author}{Larbi, S.}, \bibinfo{year}{2025}.
\newblock \bibinfo{title}{Physics informed neural network with {Fourier}
  feature for natural convection problems}.
\newblock \bibinfo{journal}{Engineering Applications of Artificial
  Intelligence} \bibinfo{volume}{146}, \bibinfo{pages}{110327}.
\bibitem[{Bradbury et~al.(2018)Bradbury, Frostig, Hawkins, Johnson, Leary,
  Maclaurin, Necula, Paszke, Vander{P}las, Wanderman-{M}ilne and
  Zhang}]{jax2018github}
\bibinfo{author}{Bradbury, J.}, \bibinfo{author}{Frostig, R.},
  \bibinfo{author}{Hawkins, P.}, \bibinfo{author}{Johnson, M.J.},
  \bibinfo{author}{Leary, C.}, \bibinfo{author}{Maclaurin, D.},
  \bibinfo{author}{Necula, G.}, \bibinfo{author}{Paszke, A.},
  \bibinfo{author}{Vander{P}las, J.}, \bibinfo{author}{Wanderman-{M}ilne, S.},
  \bibinfo{author}{Zhang, Q.}, \bibinfo{year}{2018}.
\newblock \bibinfo{title}{{JAX}: composable transformations of
  {P}ython+{N}um{P}y programs}.
\newblock \URLprefix \url{http://github.com/google/jax}.
\bibitem[{Cai et~al.(2020)Cai, Li and Liu}]{cai2020phase}
\bibinfo{author}{Cai, W.}, \bibinfo{author}{Li, X.}, \bibinfo{author}{Liu, L.},
  \bibinfo{year}{2020}.
\newblock \bibinfo{title}{A phase shift deep neural network for high frequency
  approximation and wave problems}.
\newblock \bibinfo{journal}{SIAM Journal on Scientific Computing}
  \bibinfo{volume}{42}, \bibinfo{pages}{A3285--A3312}.
\bibitem[{Cai et~al.(2024)Cai, Zhu, Shen, Wang and Cao}]{cai2024batch}
\bibinfo{author}{Cai, Z.}, \bibinfo{author}{Zhu, H.}, \bibinfo{author}{Shen,
  Q.}, \bibinfo{author}{Wang, X.}, \bibinfo{author}{Cao, X.},
  \bibinfo{year}{2024}.
\newblock \bibinfo{title}{Batch normalization alleviates the spectral bias in
  coordinate networks}, in: \bibinfo{booktitle}{Proceedings of the IEEE/CVF
  conference on computer vision and pattern recognition}, pp.
  \bibinfo{pages}{25160--25171}.
\bibitem[{Cao et~al.(2021)Cao, Fang, Wu, Zhou and Gu}]{cao2021towards}
\bibinfo{author}{Cao, Y.}, \bibinfo{author}{Fang, Z.}, \bibinfo{author}{Wu,
  Y.}, \bibinfo{author}{Zhou, D.X.}, \bibinfo{author}{Gu, Q.},
  \bibinfo{year}{2021}.
\newblock \bibinfo{title}{Towards understanding the spectral bias of deep
  learning}, in: \bibinfo{booktitle}{30th International Joint Conference on
  Artificial Intelligence (IJCAI 2021)}, \bibinfo{organization}{International
  Joint Conferences on Artificial Intelligence}. pp.
  \bibinfo{pages}{2205--2211}.
\bibitem[{Chai et~al.(2024)Chai, Cao, Li, Long and Sun}]{10630853}
\bibinfo{author}{Chai, X.}, \bibinfo{author}{Cao, W.}, \bibinfo{author}{Li,
  J.}, \bibinfo{author}{Long, H.}, \bibinfo{author}{Sun, X.},
  \bibinfo{year}{2024}.
\newblock \bibinfo{title}{Overcoming the spectral bias problem of
  physics-informed neural networks in solving the frequency-domain acoustic
  wave equation}.
\newblock \bibinfo{journal}{IEEE Transactions on Geoscience and Remote Sensing}
  \bibinfo{volume}{62}, \bibinfo{pages}{1--20}.
\newblock \DOIprefix\doi{10.1109/TGRS.2024.3440471}.
\bibitem[{Deshpande et~al.(2022)Deshpande, Agarwal and
  Bhattacharya}]{deshpande2022investigations}
\bibinfo{author}{Deshpande, M.}, \bibinfo{author}{Agarwal, S.},
  \bibinfo{author}{Bhattacharya, A.K.}, \bibinfo{year}{2022}.
\newblock \bibinfo{title}{Investigations on convergence behaviour of physics
  informed neural networks across spectral ranges and derivative orders}, in:
  \bibinfo{booktitle}{2022 IEEE Symposium Series on Computational Intelligence
  (SSCI)}, \bibinfo{organization}{IEEE}. pp. \bibinfo{pages}{1172--1179}.
\bibitem[{Ding et~al.(2025)Ding, Chen, Miyake and Li}]{11045735}
\bibinfo{author}{Ding, Y.}, \bibinfo{author}{Chen, S.},
  \bibinfo{author}{Miyake, H.}, \bibinfo{author}{Li, X.}, \bibinfo{year}{2025}.
\newblock \bibinfo{title}{Physics-informed neural networks with {Fourier}
  features for seismic wavefield simulation in time-domain nonsmooth complex
  media}.
\newblock \bibinfo{journal}{IEEE Transactions on Geoscience and Remote Sensing}
  \bibinfo{volume}{63}, \bibinfo{pages}{1--13}.
\newblock \DOIprefix\doi{10.1109/TGRS.2025.3581638}.
\bibitem[{Evans(2010)}]{evans2010partial}
\bibinfo{author}{Evans, L.C.}, \bibinfo{year}{2010}.
\newblock \bibinfo{title}{Partial Differential Equations}.
  volume~\bibinfo{volume}{19}.
\newblock \bibinfo{edition}{2} ed., \bibinfo{publisher}{American Mathematical
  Society}.
\bibitem[{Fu et~al.(2024)Fu, Xu and Liu}]{fu2024physics}
\bibinfo{author}{Fu, Z.}, \bibinfo{author}{Xu, W.}, \bibinfo{author}{Liu, S.},
  \bibinfo{year}{2024}.
\newblock \bibinfo{title}{Physics-informed kernel function neural networks for
  solving partial differential equations}.
\newblock \bibinfo{journal}{Neural Networks} \bibinfo{volume}{172},
  \bibinfo{pages}{106098}.
\bibitem[{Geifman et~al.(2022)Geifman, Galun, Jacobs and
  Ronen}]{geifman2022spectral}
\bibinfo{author}{Geifman, A.}, \bibinfo{author}{Galun, M.},
  \bibinfo{author}{Jacobs, D.}, \bibinfo{author}{Ronen, B.},
  \bibinfo{year}{2022}.
\newblock \bibinfo{title}{On the spectral bias of convolutional neural tangent
  and {Gaussian} process kernels}.
\newblock \bibinfo{journal}{Advances in Neural Information Processing Systems}
  \bibinfo{volume}{35}, \bibinfo{pages}{11253--11265}.
\bibitem[{Hou et~al.(2026)Hou, Bai, Jing and lin Huang}]{HOU2026108247}
\bibinfo{author}{Hou, B.Y.}, \bibinfo{author}{Bai, Y.L.},
  \bibinfo{author}{Jing, X.T.}, \bibinfo{author}{lin Huang, C.},
  \bibinfo{year}{2026}.
\newblock \bibinfo{title}{Fourier feature-enhanced multi-layer residual
  stacking network: A novel multiscale modeling approach for physics-informed
  neural networks}.
\newblock \bibinfo{journal}{Neural Networks} \bibinfo{volume}{195},
  \bibinfo{pages}{108247}.
\newblock \URLprefix
  \url{https://www.sciencedirect.com/science/article/pii/S0893608025011281},
  \DOIprefix\doi{https://doi.org/10.1016/j.neunet.2025.108247}.
\bibitem[{Jacot et~al.(2018)Jacot, Gabriel and Hongler}]{jacot2018neural}
\bibinfo{author}{Jacot, A.}, \bibinfo{author}{Gabriel, F.},
  \bibinfo{author}{Hongler, C.}, \bibinfo{year}{2018}.
\newblock \bibinfo{title}{Neural tangent kernel: Convergence and generalization
  in neural networks}.
\newblock \bibinfo{journal}{Advances in neural information processing systems}
  \bibinfo{volume}{31}.
\bibitem[{Khodakarami et~al.(2026)Khodakarami, Oommen, Daryakenari, Beekenkamp
  and Karniadakis}]{khodakarami2026spectral}
\bibinfo{author}{Khodakarami, S.}, \bibinfo{author}{Oommen, V.},
  \bibinfo{author}{Daryakenari, N.A.}, \bibinfo{author}{Beekenkamp, M.},
  \bibinfo{author}{Karniadakis, G.E.}, \bibinfo{year}{2026}.
\newblock \bibinfo{title}{Spectral bias in physics-informed and operator
  learning: Analysis and mitigation guidelines}.
\newblock \bibinfo{journal}{arXiv preprint arXiv:2602.19265} .
\bibitem[{Kissas et~al.(2020)Kissas, Yang, Hwuang, Witschey, Detre and
  Perdikaris}]{kissas2020machine}
\bibinfo{author}{Kissas, G.}, \bibinfo{author}{Yang, Y.},
  \bibinfo{author}{Hwuang, E.}, \bibinfo{author}{Witschey, W.R.},
  \bibinfo{author}{Detre, J.A.}, \bibinfo{author}{Perdikaris, P.},
  \bibinfo{year}{2020}.
\newblock \bibinfo{title}{Machine learning in cardiovascular flows modeling:
  Predicting arterial blood pressure from non-invasive {4D} flow {MRI} data
  using physics-informed neural networks}.
\newblock \bibinfo{journal}{Computer methods in applied mechanics and
  engineering} \bibinfo{volume}{358}, \bibinfo{pages}{112623}.
\bibitem[{Lee et~al.(2019)Lee, Xiao, Schoenholz, Bahri, Novak, Sohl-Dickstein
  and Pennington}]{lee2019wide}
\bibinfo{author}{Lee, J.}, \bibinfo{author}{Xiao, L.},
  \bibinfo{author}{Schoenholz, S.}, \bibinfo{author}{Bahri, Y.},
  \bibinfo{author}{Novak, R.}, \bibinfo{author}{Sohl-Dickstein, J.},
  \bibinfo{author}{Pennington, J.}, \bibinfo{year}{2019}.
\newblock \bibinfo{title}{Wide neural networks of any depth evolve as linear
  models under gradient descent}.
\newblock \bibinfo{journal}{Advances in neural information processing systems}
  \bibinfo{volume}{32}.
\bibitem[{Li et~al.(2023)Li, Xia, Liu and Liao}]{LI2023114963}
\bibinfo{author}{Li, S.}, \bibinfo{author}{Xia, Y.}, \bibinfo{author}{Liu, Y.},
  \bibinfo{author}{Liao, Q.}, \bibinfo{year}{2023}.
\newblock \bibinfo{title}{A deep domain decomposition method based on {Fourier}
  features}.
\newblock \bibinfo{journal}{Journal of Computational and Applied Mathematics}
  \bibinfo{volume}{423}, \bibinfo{pages}{114963}.
\newblock \URLprefix
  \url{https://www.sciencedirect.com/science/article/pii/S0377042722005611},
  \DOIprefix\doi{https://doi.org/10.1016/j.cam.2022.114963}.
\bibitem[{Liu et~al.(2022)Liu, Zhang, Peng, Zhou and Yao}]{liu2022novel}
\bibinfo{author}{Liu, X.}, \bibinfo{author}{Zhang, X.}, \bibinfo{author}{Peng,
  W.}, \bibinfo{author}{Zhou, W.}, \bibinfo{author}{Yao, W.},
  \bibinfo{year}{2022}.
\newblock \bibinfo{title}{A novel meta-learning initialization method for
  physics-informed neural networks}.
\newblock \bibinfo{journal}{Neural Computing and Applications}
  \bibinfo{volume}{34}, \bibinfo{pages}{14511--14534}.
\bibitem[{Liu et~al.(2025)Liu, Gu, Yu and Qin}]{LIU2025106886}
\bibinfo{author}{Liu, Y.}, \bibinfo{author}{Gu, H.}, \bibinfo{author}{Yu, X.},
  \bibinfo{author}{Qin, P.}, \bibinfo{year}{2025}.
\newblock \bibinfo{title}{Diminishing spectral bias in physics-informed neural
  networks using spatially-adaptive {Fourier} feature encoding}.
\newblock \bibinfo{journal}{Neural Networks} \bibinfo{volume}{182},
  \bibinfo{pages}{106886}.
\newblock \URLprefix
  \url{https://www.sciencedirect.com/science/article/pii/S0893608024008153},
  \DOIprefix\doi{https://doi.org/10.1016/j.neunet.2024.106886}.
\bibitem[{Luo et~al.(2025)Luo, Zhao, Wang, Li, Wen, Liang, Soekmadji and
  Liao}]{luo2025physics}
\bibinfo{author}{Luo, K.}, \bibinfo{author}{Zhao, J.}, \bibinfo{author}{Wang,
  Y.}, \bibinfo{author}{Li, J.}, \bibinfo{author}{Wen, J.},
  \bibinfo{author}{Liang, J.}, \bibinfo{author}{Soekmadji, H.},
  \bibinfo{author}{Liao, S.}, \bibinfo{year}{2025}.
\newblock \bibinfo{title}{Physics-informed neural networks for {PDE} problems:
  A comprehensive review}.
\newblock \bibinfo{journal}{Artificial Intelligence Review}
  \bibinfo{volume}{58}, \bibinfo{pages}{323}.
\bibitem[{Luo et~al.(2022a)Luo, Ma, Wang, Xu and Zhang}]{luo2022upper}
\bibinfo{author}{Luo, T.}, \bibinfo{author}{Ma, Z.}, \bibinfo{author}{Wang,
  Z.}, \bibinfo{author}{Xu, Z.J.}, \bibinfo{author}{Zhang, Y.},
  \bibinfo{year}{2022}a.
\newblock \bibinfo{title}{An upper limit of decaying rate with respect to
  frequency in linear frequency principle model}, in:
  \bibinfo{booktitle}{Mathematical and Scientific Machine Learning},
  \bibinfo{organization}{PMLR}. pp. \bibinfo{pages}{205--214}.
\bibitem[{Luo et~al.(2019)Luo, Ma, Xu and Zhang}]{luo2019theory}
\bibinfo{author}{Luo, T.}, \bibinfo{author}{Ma, Z.}, \bibinfo{author}{Xu,
  Z.Q.J.}, \bibinfo{author}{Zhang, Y.}, \bibinfo{year}{2019}.
\newblock \bibinfo{title}{Theory of the frequency principle for general deep
  neural networks}.
\newblock \bibinfo{journal}{arXiv preprint arXiv:1906.09235} .
\bibitem[{Luo et~al.(2022b)Luo, Ma, Xu and Zhang}]{luo2022exact}
\bibinfo{author}{Luo, T.}, \bibinfo{author}{Ma, Z.}, \bibinfo{author}{Xu,
  Z.Q.J.}, \bibinfo{author}{Zhang, Y.}, \bibinfo{year}{2022}b.
\newblock \bibinfo{title}{On the exact computation of linear frequency
  principle dynamics and its generalization}.
\newblock \bibinfo{journal}{SIAM Journal on Mathematics of Data Science}
  \bibinfo{volume}{4}, \bibinfo{pages}{1272--1292}.
\bibitem[{Ma et~al.(2021)Ma, Xu and Zhang}]{ma2021frequency}
\bibinfo{author}{Ma, Y.}, \bibinfo{author}{Xu, Z.Q.J.}, \bibinfo{author}{Zhang,
  J.}, \bibinfo{year}{2021}.
\newblock \bibinfo{title}{Frequency principle in deep learning beyond
  gradient-descent-based training}.
\newblock \bibinfo{journal}{arXiv preprint arXiv:2101.00747} .
\bibitem[{Molina et~al.(2026)Molina, Petrache, Costabal and
  Courdurier}]{molina2026optimizing}
\bibinfo{author}{Molina, J.}, \bibinfo{author}{Petrache, M.},
  \bibinfo{author}{Costabal, F.S.}, \bibinfo{author}{Courdurier, M.},
  \bibinfo{year}{2026}.
\newblock \bibinfo{title}{Optimizing the dynamics of the frequency bias in
  {Fourier} features neural networks}.
\newblock \bibinfo{journal}{Neurocomputing} \bibinfo{volume}{697},
  \bibinfo{pages}{134167}.
\bibitem[{Rahaman et~al.(2019)Rahaman, Baratin, Arpit, Draxler, Lin, Hamprecht,
  Bengio and Courville}]{rahaman2019spectral}
\bibinfo{author}{Rahaman, N.}, \bibinfo{author}{Baratin, A.},
  \bibinfo{author}{Arpit, D.}, \bibinfo{author}{Draxler, F.},
  \bibinfo{author}{Lin, M.}, \bibinfo{author}{Hamprecht, F.},
  \bibinfo{author}{Bengio, Y.}, \bibinfo{author}{Courville, A.},
  \bibinfo{year}{2019}.
\newblock \bibinfo{title}{On the spectral bias of neural networks}, in:
  \bibinfo{booktitle}{International Conference on Machine Learning},
  \bibinfo{organization}{PMLR}. pp. \bibinfo{pages}{5301--5310}.
\bibitem[{Raissi et~al.(2019)Raissi, Perdikaris and
  Karniadakis}]{raissi2019physics}
\bibinfo{author}{Raissi, M.}, \bibinfo{author}{Perdikaris, P.},
  \bibinfo{author}{Karniadakis, G.E.}, \bibinfo{year}{2019}.
\newblock \bibinfo{title}{Physics-informed neural networks: A deep learning
  framework for solving forward and inverse problems involving nonlinear
  partial differential equations}.
\newblock \bibinfo{journal}{Journal of Computational physics}
  \bibinfo{volume}{378}, \bibinfo{pages}{686--707}.
\bibitem[{Raissi et~al.(2020)Raissi, Yazdani and
  Karniadakis}]{raissi2020hidden}
\bibinfo{author}{Raissi, M.}, \bibinfo{author}{Yazdani, A.},
  \bibinfo{author}{Karniadakis, G.E.}, \bibinfo{year}{2020}.
\newblock \bibinfo{title}{Hidden fluid mechanics: Learning velocity and
  pressure fields from flow visualizations}.
\newblock \bibinfo{journal}{Science} \bibinfo{volume}{367},
  \bibinfo{pages}{1026--1030}.
\bibitem[{Ramasinghe and Lucey(2022)}]{ramasinghe2022beyond}
\bibinfo{author}{Ramasinghe, S.}, \bibinfo{author}{Lucey, S.},
  \bibinfo{year}{2022}.
\newblock \bibinfo{title}{Beyond periodicity: Towards a unifying framework for
  activations in coordinate-mlps}, in: \bibinfo{booktitle}{European Conference
  on Computer Vision}, \bibinfo{organization}{Springer}. pp.
  \bibinfo{pages}{142--158}.
\bibitem[{Ronen et~al.(2019)Ronen, Jacobs, Kasten and
  Kritchman}]{ronen2019convergence}
\bibinfo{author}{Ronen, B.}, \bibinfo{author}{Jacobs, D.},
  \bibinfo{author}{Kasten, Y.}, \bibinfo{author}{Kritchman, S.},
  \bibinfo{year}{2019}.
\newblock \bibinfo{title}{The convergence rate of neural networks for learned
  functions of different frequencies}.
\newblock \bibinfo{journal}{Advances in Neural Information Processing Systems}
  \bibinfo{volume}{32}.
\bibitem[{Sahli~Costabal et~al.(2020)Sahli~Costabal, Yang, Perdikaris, Hurtado
  and Kuhl}]{sahli2020physics}
\bibinfo{author}{Sahli~Costabal, F.}, \bibinfo{author}{Yang, Y.},
  \bibinfo{author}{Perdikaris, P.}, \bibinfo{author}{Hurtado, D.E.},
  \bibinfo{author}{Kuhl, E.}, \bibinfo{year}{2020}.
\newblock \bibinfo{title}{Physics-informed neural networks for cardiac
  activation mapping}.
\newblock \bibinfo{journal}{Frontiers in Physics} \bibinfo{volume}{8},
  \bibinfo{pages}{42}.
\bibitem[{Sallam and F{\"u}rth(2023)}]{sallam2023use}
\bibinfo{author}{Sallam, O.}, \bibinfo{author}{F{\"u}rth, M.},
  \bibinfo{year}{2023}.
\newblock \bibinfo{title}{On the use of {Fourier} features-physics informed
  neural networks ({FF-PINN}) for forward and inverse fluid mechanics
  problems}.
\newblock \bibinfo{journal}{Proceedings of the Institution of Mechanical
  Engineers, Part M: Journal of Engineering for the Maritime Environment}
  \bibinfo{volume}{237}, \bibinfo{pages}{846--866}.
\bibitem[{Saragadam et~al.(2023)Saragadam, LeJeune, Tan, Balakrishnan,
  Veeraraghavan and Baraniuk}]{saragadam2023wire}
\bibinfo{author}{Saragadam, V.}, \bibinfo{author}{LeJeune, D.},
  \bibinfo{author}{Tan, J.}, \bibinfo{author}{Balakrishnan, G.},
  \bibinfo{author}{Veeraraghavan, A.}, \bibinfo{author}{Baraniuk, R.G.},
  \bibinfo{year}{2023}.
\newblock \bibinfo{title}{{WIRE}: Wavelet implicit neural representations}, in:
  \bibinfo{booktitle}{Proceedings of the IEEE/CVF Conference on Computer Vision
  and Pattern Recognition}, pp. \bibinfo{pages}{18507--18516}.
\bibitem[{Sitzmann et~al.(2020)Sitzmann, Martel, Bergman, Lindell and
  Wetzstein}]{sitzmann2020implicit}
\bibinfo{author}{Sitzmann, V.}, \bibinfo{author}{Martel, J.},
  \bibinfo{author}{Bergman, A.}, \bibinfo{author}{Lindell, D.},
  \bibinfo{author}{Wetzstein, G.}, \bibinfo{year}{2020}.
\newblock \bibinfo{title}{Implicit neural representations with periodic
  activation functions}.
\newblock \bibinfo{journal}{Advances in neural information processing systems}
  \bibinfo{volume}{33}, \bibinfo{pages}{7462--7473}.
\bibitem[{Tancik et~al.(2020)Tancik, Srinivasan, Mildenhall, Fridovich-Keil,
  Raghavan, Singhal, Ramamoorthi, Barron and Ng}]{tancik2020fourier}
\bibinfo{author}{Tancik, M.}, \bibinfo{author}{Srinivasan, P.},
  \bibinfo{author}{Mildenhall, B.}, \bibinfo{author}{Fridovich-Keil, S.},
  \bibinfo{author}{Raghavan, N.}, \bibinfo{author}{Singhal, U.},
  \bibinfo{author}{Ramamoorthi, R.}, \bibinfo{author}{Barron, J.},
  \bibinfo{author}{Ng, R.}, \bibinfo{year}{2020}.
\newblock \bibinfo{title}{Fourier features let networks learn high frequency
  functions in low dimensional domains}.
\newblock \bibinfo{journal}{Advances in Neural Information Processing Systems}
  \bibinfo{volume}{33}, \bibinfo{pages}{7537--7547}.
\bibitem[{Tarbiyati and Nemati~Saray(2024)}]{tarbiyati2024weight}
\bibinfo{author}{Tarbiyati, H.}, \bibinfo{author}{Nemati~Saray, B.},
  \bibinfo{year}{2024}.
\newblock \bibinfo{title}{Weight initialization algorithm for physics-informed
  neural networks using finite differences}.
\newblock \bibinfo{journal}{Engineering with Computers} \bibinfo{volume}{40},
  \bibinfo{pages}{1603--1619}.
\bibitem[{Vasheghani~Farahani et~al.(2026)Vasheghani~Farahani, Tahsini,
  Nematollahi and Ghiasi~Afzal}]{vasheghani2026novel}
\bibinfo{author}{Vasheghani~Farahani, H.}, \bibinfo{author}{Tahsini, A.M.},
  \bibinfo{author}{Nematollahi, A.}, \bibinfo{author}{Ghiasi~Afzal, M.},
  \bibinfo{year}{2026}.
\newblock \bibinfo{title}{A novel physics-informed neural network architecture
  for solving nonlinear {PDEs} in fluid mechanics without labeled data}.
\newblock \bibinfo{journal}{Engineering with Computers} \bibinfo{volume}{42},
  \bibinfo{pages}{46}.
\bibitem[{Wang et~al.(2026a)Wang, Koohy, Lu and Perdikaris}]{wang2026pinns}
\bibinfo{author}{Wang, S.}, \bibinfo{author}{Koohy, S.}, \bibinfo{author}{Lu,
  Y.}, \bibinfo{author}{Perdikaris, P.}, \bibinfo{year}{2026}a.
\newblock \bibinfo{title}{When {PINNs} go wrong: Pseudo-time stepping against
  spurious solutions}.
\newblock \bibinfo{journal}{arXiv preprint arXiv:2604.23528} .
\bibitem[{Wang et~al.(2023)Wang, Sankaran, Wang and
  Perdikaris}]{wang2023expert}
\bibinfo{author}{Wang, S.}, \bibinfo{author}{Sankaran, S.},
  \bibinfo{author}{Wang, H.}, \bibinfo{author}{Perdikaris, P.},
  \bibinfo{year}{2023}.
\newblock \bibinfo{title}{An expert's guide to training physics-informed neural
  networks}.
\newblock \bibinfo{journal}{arXiv preprint arXiv:2308.08468} .
\bibitem[{Wang et~al.(2021)Wang, Wang and Perdikaris}]{wang2021eigenvector}
\bibinfo{author}{Wang, S.}, \bibinfo{author}{Wang, H.},
  \bibinfo{author}{Perdikaris, P.}, \bibinfo{year}{2021}.
\newblock \bibinfo{title}{On the eigenvector bias of {Fourier} feature
  networks: From regression to solving multi-scale {PDEs} with physics-informed
  neural networks}.
\newblock \bibinfo{journal}{Computer Methods in Applied Mechanics and
  Engineering} \bibinfo{volume}{384}, \bibinfo{pages}{113938}.
\bibitem[{Wang et~al.(2022)Wang, Yu and Perdikaris}]{wang2022and}
\bibinfo{author}{Wang, S.}, \bibinfo{author}{Yu, X.},
  \bibinfo{author}{Perdikaris, P.}, \bibinfo{year}{2022}.
\newblock \bibinfo{title}{When and why {PINNs} fail to train: A neural tangent
  kernel perspective}.
\newblock \bibinfo{journal}{Journal of Computational Physics}
  \bibinfo{volume}{449}, \bibinfo{pages}{110768}.
\bibitem[{Wang et~al.(2026b)Wang, Du, Li and Luo}]{wang2026structural}
\bibinfo{author}{Wang, X.}, \bibinfo{author}{Du, K.}, \bibinfo{author}{Li, Z.},
  \bibinfo{author}{Luo, H.}, \bibinfo{year}{2026}b.
\newblock \bibinfo{title}{Structural natural frequency-embedded {Fourier}
  feature physics-informed neural networks for earthquake dynamic response
  prediction}.
\newblock \bibinfo{journal}{Advanced Modeling and Simulation in Engineering
  Sciences} \bibinfo{volume}{13}, \bibinfo{pages}{6}.
\bibitem[{de~Wolff et~al.(2021)de~Wolff, Carrillo, Mart{\'\i} and
  Sanchez-Pi}]{de2021assessing}
\bibinfo{author}{de~Wolff, T.}, \bibinfo{author}{Carrillo, H.},
  \bibinfo{author}{Mart{\'\i}, L.}, \bibinfo{author}{Sanchez-Pi, N.},
  \bibinfo{year}{2021}.
\newblock \bibinfo{title}{Assessing physics informed neural networks in ocean
  modelling and climate change applications}, in: \bibinfo{booktitle}{AI:
  Modeling Oceans and Climate Change Workshop at ICLR 2021}.
\bibitem[{Wu et~al.(2025)Wu, Aguiar, Johansson and Barreau}]{wu2025iterative}
\bibinfo{author}{Wu, Y.}, \bibinfo{author}{Aguiar, M.},
  \bibinfo{author}{Johansson, K.H.}, \bibinfo{author}{Barreau, M.},
  \bibinfo{year}{2025}.
\newblock \bibinfo{title}{Iterative training of physics-informed neural
  networks with {Fourier}-enhanced features}.
\newblock \bibinfo{journal}{arXiv preprint arXiv:2510.19399} .
\bibitem[{Wu et~al.(2026)Wu, Yuan, Gao, Yang and Hou}]{wu2026novel}
\bibinfo{author}{Wu, Y.}, \bibinfo{author}{Yuan, H.}, \bibinfo{author}{Gao,
  Y.}, \bibinfo{author}{Yang, S.}, \bibinfo{author}{Hou, Y.},
  \bibinfo{year}{2026}.
\newblock \bibinfo{title}{A novel {Fourier} feature physics-informed neural
  networks based on the boundary element method for solving scattering of {SH}
  wave induced by complex topography}.
\newblock \bibinfo{journal}{Earthquake Engineering \& Structural Dynamics} .
\bibitem[{Xu(2018)}]{xu2018understanding}
\bibinfo{author}{Xu, Z.J.}, \bibinfo{year}{2018}.
\newblock \bibinfo{title}{Understanding training and generalization in deep
  learning by {Fourier} analysis}.
\newblock \bibinfo{journal}{arXiv preprint arXiv:1808.04295} .
\bibitem[{Xu et~al.(2025)Xu, Zhang and Cai}]{XU2025113905}
\bibinfo{author}{Xu, Z.Q.J.}, \bibinfo{author}{Zhang, L.},
  \bibinfo{author}{Cai, W.}, \bibinfo{year}{2025}.
\newblock \bibinfo{title}{On understanding and overcoming spectral biases of
  deep neural network learning methods for solving {PDEs}}.
\newblock \bibinfo{journal}{Journal of Computational Physics}
  \bibinfo{volume}{530}, \bibinfo{pages}{113905}.
\newblock \URLprefix
  \url{https://www.sciencedirect.com/science/article/pii/S0021999125001883},
  \DOIprefix\doi{https://doi.org/10.1016/j.jcp.2025.113905}.
\bibitem[{Yang and Hu(2021)}]{yang2021tensor}
\bibinfo{author}{Yang, G.}, \bibinfo{author}{Hu, E.J.}, \bibinfo{year}{2021}.
\newblock \bibinfo{title}{Tensor programs iv: Feature learning in
  infinite-width neural networks}, in: \bibinfo{booktitle}{International
  Conference on Machine Learning}, \bibinfo{organization}{PMLR}. pp.
  \bibinfo{pages}{11727--11737}.
\bibitem[{Zhang et~al.(2021)Zhang, Luo, Ma and Xu}]{zhang2021linear}
\bibinfo{author}{Zhang, Y.}, \bibinfo{author}{Luo, T.}, \bibinfo{author}{Ma,
  Z.}, \bibinfo{author}{Xu, Z.Q.J.}, \bibinfo{year}{2021}.
\newblock \bibinfo{title}{A linear frequency principle model to understand the
  absence of overfitting in neural networks}.
\newblock \bibinfo{journal}{Chinese Physics Letters} \bibinfo{volume}{38},
  \bibinfo{pages}{038701}.

\end{thebibliography}

\appendix

\section{Derivation of the FB equation for $\widehat{r}$}\label{app:Dynamics}

The following computations are based on the theory of tempered distributions. Recall that, over $\mathbb R^d$, Schwartz functions are denoted by $\mathcal S(\mathbb R^d)$ and tempered distributions form the dual space, denoted $\mathcal S'(\mathbb R^d)$. As customary, we say that a relation holds in the sense of distributions over a space $\mathbb R^d$, provided it holds in duality with Schwartz from $\mathcal S(\mathbb R^d)$.

We now give the proofs of section \ref{sec:FBeq}

\begin{proof}[Proof of Lemma \ref{thm:Lemma}]
Let be $\phi \in \mathcal{S}$ \begin{align} &\left\langle\mathcal{F}\left[\frac{d}{dt}r\right],\phi\right\rangle = \left\langle \frac{d}{dt}r,\mathcal{F} [\phi]\right\rangle \\ = &\int_{\R^{d}} \overline{\left(\int_{\R^{d}} - K\left(\bsx,\bsxp\right)(\mathcal{D}_{x'}^{*} \mathcal{D}_{x'} ) r_{\mu}\left(\bsxp,\Theta\right)d\bsxp \right)} \int_{\R^{d}} \phi(\bxi) e^{-2\pi i\bsx \cdot \bxi} d\bxi d\bsx \nonumber \\ = & -\mathbb{E}_{\theta}\int_{\R^{3d}} \overline{\nabla_{\theta} f^0}(\bsx,\theta)\cdot\overline{\nabla_{\theta} f^0}(\bsxp,\theta)\overline{(\mathcal{D}_{x'}^{*} \mathcal{D}_{x'} ) r_{\mu}}\left(\bsxp,\Theta\right) \phi(\bxi) e^{-2\pi i\bsx \cdot \bxi} d\bsxp d\bxi d\bsx \nonumber \\ = & -\mathbb{E}_{\theta}\int_{\R^{d}} \overline{\nabla_{\theta} f^0}(\bsxp,\theta)\overline{(\mathcal{D}_{x'}^{*} \mathcal{D}_{x'}) r_{\mu}}\left(\bsxp,\Theta\right) d\bsxp \cdot \int_{\R^{d}} \left( \int_{\R^{d}}\overline{\nabla_{\theta} f^0}(\bsx,\theta) e^{-2\pi i\bsx \cdot \bxi} d\bsx \right) \phi(\bxi) d\bxi. \nonumber \end{align} Applying the Parserval theorem in the first integral, we arrived at \begin{align} &-\left\langle\mathcal{F}\left[\frac{d}{dt}r\right],\phi\right\rangle \\ = & \mathbb{E}_{\theta}\int_{\R^{d}} \mathcal{F}_{\bsxp\rightarrow\bxi^{\prime}}\left[\nabla_{\theta} f^0\right](\bxip,\theta) \overline{\mathcal{F}_{\bsxp\rightarrow\bxi^{\prime}} \left[(\mathcal{D}_{x'}^{*} \mathcal{D}_x' ) r_{\mu}\right]}\left(\bxip,\Theta\right) d\bsxp \cdot \int_{\R^{d}} \overline{\mathcal{F}_{\bsx\rightarrow\bxi}\left[\nabla_{\theta} f^0\right]}(\bxi,\theta) \phi(\bxi) d\bxi \nonumber \\ = & \mathbb{E}_{\theta}\int_{\R^{d}} \overline{\left(\int_{\R^{d}} \overline{\mathcal{F}_{\bsxp\rightarrow\bxi^{\prime}}\left[\nabla_{\theta} f^0\right](\bxip,\theta)} \cdot \mathcal{F}_{\bsx\rightarrow\bxi}\left[\nabla_{\theta} f^0\right](\bxi,\theta) p(\bxip) \widehat{r}_{\mu}\left(\bxip,\Theta\right) d\bsxp \right)} \phi(\bxi) d\bxi \nonumber \\ = & \mathbb{E}_{\theta}\int_{\R^{d}} \overline{\left(\int_{\R^{d}} \widehat{K}\left(\bxi,\bxip\right)p(\bxip) \widehat{r}_{\mu}\left(\bxip,\Theta\right) d\bsxp \right)} \phi(\bxi) d\bxi = -\left\langle \mathcal{P}[\widehat{r}_{\mu}], \phi \right\rangle. \nonumber \end{align}
\end{proof}

\begin{proof}[Proof of Theorem \ref{thm:frequency-dynamics}]
The proof follows the same steps as in \citet{molina2026optimizing}, with the only modification being the presence of the multiplicative factor $p(\cdot)$ arising from the operator $\mathcal{D}_x^*\mathcal{D}_x$. Namely, \citet{molina2026optimizing} found that in the case $\mathcal{D}_x:=\mathcal{I}$ (a.k.a the loss function is for a regression task), the dynamics is given by: 
\begin{equation}\label{eq:dynamicGeneral}
    \left\langle \partial_t\widehat{r},\psi\right\rangle = -\int_{\mathbb{R}^{d}}   \overline{\psi(\bxi)}  \left(\int_{\mathbb{R}^{d}}\widehat{K}\left(\bxi,\bxip\right)\widehat{r}_{\rho_{\text{data}}}\left(\bxip,\Theta\right)d\bxip\right), \quad \forall \psi\in \mathcal S(\mathbb R^d)
\end{equation}
where $\widehat{r}_{\rho_{\text{data}}}=\mathcal{F}[r \cdot \rho_{\text{data}}]$ and 
\begin{equation}
\widehat{K}\left(\bxi,\bxip\right) 
:=  \mathbb{E}_{\theta\sim\rho_a\times\rho_{\bsw}}\left[\overline{\mathcal{F}_{\bsxp\rightarrow\bxi^{\prime}}\left[\nabla_\theta f^0(\bsxp,\theta)\right]}(\bxi^{\prime})\cdot\mathcal{F}_{\bsx\rightarrow\bxi}\left[\nabla_\theta f^0(\bsx,\theta)\right](\bxi)  \right]
\end{equation}
and making explicit $\widehat{K}$ for the FF network they arrive to:
\begin{alignat}{1}\label{eq:dynamicsfinal}
\left\langle \frac{d}{dt}\widehat{r},\psi\right\rangle 
= & -\frac{\sigma_{a}^{2}}{4\pi^{2}} \int_{\mathbb{R}^{d}}   \overline{\nabla\psi(\bxi)} \cdot \int_{\mathbb{R}^{2}} \nabla{\widehat{r}_{\rho_{\text{data}}}\left(\frac{y'}{y}\bxi\right)} \rho_{\bsw}\left(\frac{1}{y}\bxi\right)   \frac{\overline{\widehat{g'}\left(y'\right)} \widehat{g'}\left(y\right)}{\lvert y \rvert^d}dy'dy d\bxi  \\
 & -\int_{\mathbb{R}^{d}}   \overline{\psi(\bxi)} \int_{\mathbb{R}^{2}} {\widehat{r}_{\rho_{\text{data}}}\left(\frac{y'}{y}\bxi\right)} \rho_{\bsw}\left(\frac{ 1}{ y }\bxi\right)  \frac{\overline{\widehat{g}\left(y'\right)} \widehat{g}\left(y\right)}{\lvert y \rvert^d}dy'dy d\bxi,  \nonumber
\end{alignat}
Then, as in the PINNs setting, the dynamics of \eqref{eq:dynamicGeneral} only differ in the multiplicative symbol factor $p$ we obtain:
\begin{align}\label{eq:dynamicsfinalPINNs}
\left\langle \partial_t \widehat{r}, \psi \right\rangle
= & -\frac{\sigma_a^2}{4\pi^2}\int_{\mathbb{R}^d}\overline{\nabla \psi(\bxi)} \cdot \int_{\mathbb{R}^2}\nabla\Big( p\Big(\frac{y'}{y}\bxi\Big) \widehat{r}_\mu\Big(\frac{y'}{y}\bxi,\Theta\Big)\Big) \rho_{\bsw}\Big(\frac{1}{y}\bxi\Big) \frac{\overline{\widehat{g'}(y')} \widehat{g'}(y)}{|y|^d}\, dy' dy \, d\bxi \nonumber \\
& - \int_{\mathbb{R}^d}\overline{\psi(\bxi)} \int_{\mathbb{R}^2} p\Big(\frac{y'}{y}\bxi\Big) \widehat{r}_\mu\Big(\frac{y'}{y}\bxi,\Theta\Big) \rho_{\bsw}\Big(\frac{1}{y}\bxi\Big) \frac{\overline{\widehat{g}(y')} \widehat{g}(y)}{|y|^d} \, dy' dy \, d\bxi.
\end{align}

\end{proof}

\begin{proof}[Proof of Corollary \ref{cor:Nogeneral_PINNs}]
As the FF network is a sum of general NN (equation \eqref{eq:NNgeneral}), then the NTK is
\begin{alignat}{1}
 K(\bsx,\bsxp)  = & \mathbb{E}_\theta\nabla_{\theta}f^0(\bsx,\theta)\cdot\nabla_{\theta}f^0(\bsxp,\theta)\\
 = & \mathbb{E}_\theta\left[a^{2}g_{1}'\left(\bsw\cdot \bsx\right)g_{1}'\left(\bsw\cdot \bsxp\right)\bsx\cdot \bsxp\right]  + \mathbb{E}_\theta \left[b^{2}g_{2}'\left(\bsw\cdot \bsx\right)g_{2}'\left(\bsw\cdot \bsxp\right)\bsx\cdot \bsxp\right] \nonumber\\
 & +\mathbb{E}_\theta\left[g_{1}\left(\bsw\cdot \bsx\right)g_{1}\left(\bsw\cdot \bsxp\right)\right]+ \mathbb{E}_\theta\left[g_{2}\left(\bsw\cdot \bsx\right)g_{2}\left(\bsw\cdot \bsxp\right)\right]. \nonumber
\end{alignat}
and using that
\begin{alignat}{2}
    \widehat{g}_1(z) = & \frac{1}{2}\left(\delta(z-1)+\delta(z+1)\right)
    \quad \text{and} \quad & \widehat{g_2}(z) = & \frac{1}{2i}\left(\delta(z-1)-\delta(z+1)\right).
\end{alignat}
Then, computing directly the equation \eqref{eq:dynamicsfinalPINNs} it follows similar to the steps as in \citet{molina2026optimizing}.
\end{proof}

\section{Effect of PDE-symbol mismatch}\label{app:transport_symbol_mismatch}

To further assess the role of the PDE information in the proposed initialization, we consider a transport example, for which the true transport velocity is $c=1$. We compare the standard Gaussian initialization with four PDE-informed initializations. The first is constructed using the correct transport velocity, $c=1$, while the remaining three use deliberately incorrect values, $c=2$, $c=5$, and $c=10$.

Importantly, the transport equation used to define the training loss is identical in all experiments and always uses the true value $c=1$. Only the polynomial employed to construct the PDE-informed initialization is changed. Using an incorrect value of $c$ therefore rotates this preferred region in frequency space away from the characteristic set of the actual PDE.

Figure~\ref{fig:transport_wrong_c} shows the training loss and the relative $L^2$ prediction error for the different initializations. The PDE-informed initialization, with the correct value $c=1$, achieves substantially faster training dynamics and lower prediction error than all alternatives. When $c=2$ is used to construct the initialization, the resulting performance is already much closer to that of the standard Gaussian initialization. As the assumed velocity deviates further from the true value, with $c=5$ and $c=10$, both the optimization dynamics and prediction accuracy progressively deteriorate.

These results provide additional evidence that the improvement produced by the proposed initialization is specifically associated with the PDE-dependent frequency geometry encoded by $p_{\mathrm{pde}}$, rather than merely with using a non-Gaussian frequency distribution. In particular, among the tested initializations, the best performance is obtained when the polynomial used for sampling matches the true differential operator, while increasing the mismatch between the assumed and true characteristic sets leads to progressively worse dynamics.

\begin{figure}[ht]
    \centering    \includegraphics[width=.9\textwidth]{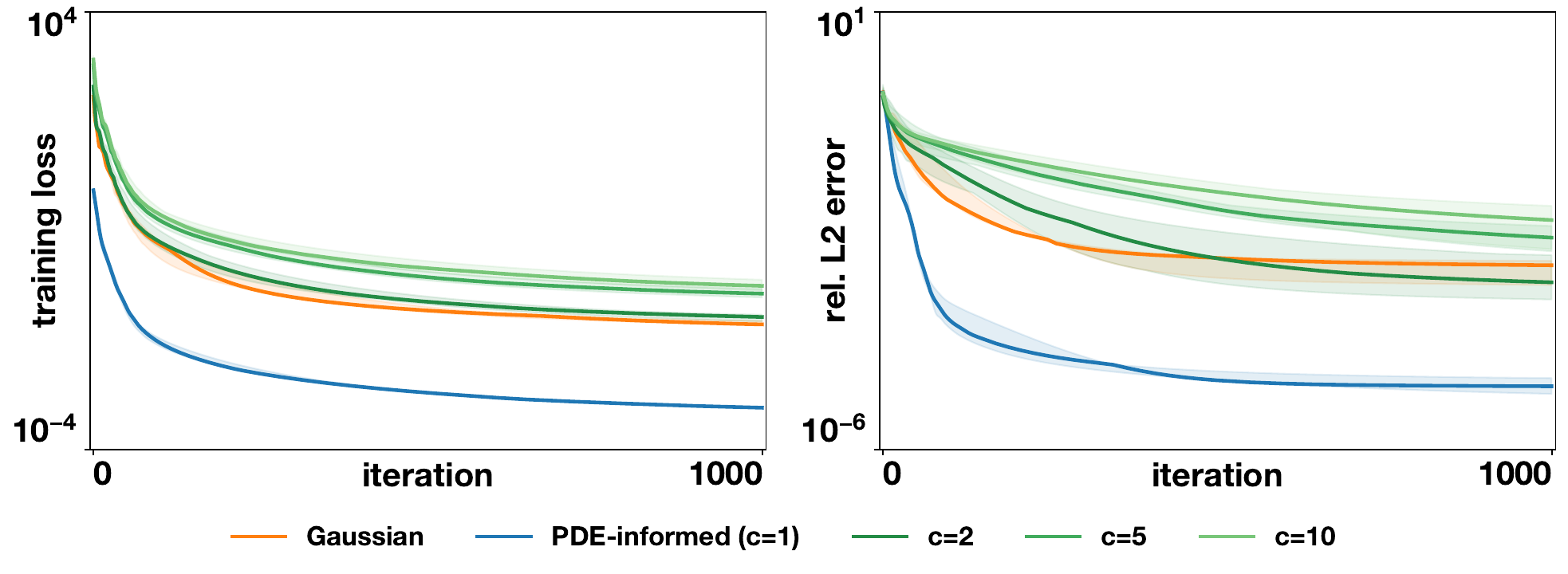}
    \caption{Effect of using an inaccurate PDE symbol in the PDE-informed initialization for the \emph{Transport -- Localized Gaussian} example of Table~\ref{tab:examples}, whose true transport velocity is $c=1$. The left panel shows the training loss and the right panel the relative $L^2$ prediction error. The standard Gaussian initialization is shown in orange, while the PDE-informed initialization constructed using the correct velocity $c=1$ is shown in blue. PDE-informed initializations constructed using the incorrect values $c=2$, $c=5$, and $c=10$ are shown using progressively lighter shades of green. The PDE used during training always has $c=1$; only the polynomial used to construct the initialization density is modified. Initialization based on the correct PDE symbol yields the best dynamics among the tested alternatives, while performance progressively deteriorates as the assumed transport velocity deviates further from its true value.}
    \label{fig:transport_wrong_c}
\end{figure}

\section{Other JAXPI benchmarking examples} 

\subsection{Principal symbol of the steady lid-driven cavity equations}

The steady incompressible Navier--Stokes equations are
\begin{align}
    u u_x+v u_y+p_x-\frac1{Re}\Delta u &=0, \nonumber \\
    u v_x+v v_y+p_y-\frac1{Re}\Delta v &=0,\\
    u_x+v_y &=0. \nonumber
\end{align}
Since the proposed initialization depends only on the principal linear operator, the nonlinear convective terms are discarded, yielding the steady Stokes operator
\begin{align}
    p_x-\frac1{Re}\Delta u &=0, \nonumber\\
    p_y-\frac1{Re}\Delta v &=0, \\
    u_x+v_y &=0. \nonumber
\end{align}
Therefore, the Fourier-transformed system is
\begin{align}
    \frac{4\pi^2}{Re}    (\xi_x^2+\xi_y^2)\hat{u} + 2\pi i\xi_x\hat{p} &=0, \nonumber \\
    \frac{4\pi^2}{Re}   (\xi_x^2+\xi_y^2)\hat{v} + 2\pi i\xi_y\hat{p} &=0,  \\
    2\pi i\xi_x\hat{u} + 2\pi i\xi_y\hat{v} &=0. \nonumber
\end{align}
Hence, the Fourier symbol of the principal operator is
\begin{equation}
\widehat{D}(\xi_x,\xi_y) =
\begin{bmatrix}
    \dfrac{4\pi^2}{Re}(\xi_x^2+\xi_y^2) & 0 & 2\pi i\xi_x \nonumber \\
    0  & \dfrac{4\pi^2}{Re}(\xi_x^2+\xi_y^2) & 2\pi i\xi_y  \\
    2\pi i\xi_x & 2\pi i\xi_y & 0 \nonumber
\end{bmatrix}.
\end{equation}
Applying the Frobenius norm, $p_{\mathrm{pde}}(\bxi)=\|\widehat{D}\|_F^2$, yields
\begin{align}
    p_{\mathrm{pde}}(\bxi) &= 2\left(   \frac{4\pi^2}{Re}    (\xi_x^2+\xi_y^2)\right)^2 + 2(2\pi\xi_x)^2 + 2(2\pi\xi_y)^2\\
    &= \frac{32\pi^4}{Re^2}   (\xi_x^2+\xi_y^2)^2 +  8\pi^2    (\xi_x^2+\xi_y^2). \nonumber
\end{align}
Then, in the regimen of high Reynolds number, the coefficient of the quartic term satisfies $\frac{32\pi^4}{Re^2}\ll1$. Consequently, our proposed initialization is close to
\begin{equation}
    \rho(\bsw) \approx  \exp\!\left(    -8\pi^2\lambda \sigma_{\bsw} \|\bsw\|^2  \right)/C,
\end{equation}
which is exactly a Gaussian density with normalization constant $C$. Therefore, for sufficiently large Reynolds numbers, the PDE-informed initialization naturally predicts that an isotropic Gaussian initialization is asymptotically optimal. As the viscous contribution vanishes, the principal symbol is dominated by the incompressibility constraint, resulting in a Gaussian distribution.

\subsection{Ablation study using PDE-informed initialization}

To further evaluate the robustness of the proposed initialization, we performed an ablation study using two benchmark problems from the JAX-PI repository: the Advection and Allen--Cahn equations. Tables~\ref{tab:ablation_adv} and~\ref{tab:ablation_allen} compare the PDE-informed initialization under different combinations of the optimization strategies available in JAX-PI. The results indicate that these additional techniques do not necessarily improve the performance of the proposed initialization, suggesting that the PDE-informed sampling already provides a favorable starting point for optimization. Among the tested components, gradient normalization appears to be the most important, consistently yielding faster convergence while maintaining competitive prediction errors. In contrast, the time-period strategy is not compatible with the proposed initialization for the Advection equation, leading to a substantial degradation in predictive accuracy when enabled. Overall, these experiments suggest that the PDE-informed initialization can be effectively combined with existing PINN training strategies, although not all combinations are complementary and some heuristics may interfere with the optimization dynamics induced by the proposed frequency-aware initialization.

\begin{table}[ht]
\centering
\caption{Ablation study for the Advection equation.}
\label{tab:ablation_adv}
\begin{tabular}{cccc|cc}
\hline
\multicolumn{4}{c|}{\textbf{Ablation Settings}} &
\multicolumn{2}{c}{\textbf{Performance}} \\
\hline
\textbf{Time Period} &
\textbf{Grad Norm} & \textbf{RWF} & \textbf{Causal Rel.} &
\textbf{L2 Error} & \textbf{Run Time (min)} \\
\hline
\checkmark & \checkmark & \checkmark & \checkmark & 3.91e-04 & 22.01 \\
$\times$ & \checkmark & \checkmark & $\times$    & 4.89e-05 & 22.9 \\
$\times$ & \checkmark & $\times$    & \checkmark & 6.39e-05 & 21.94 \\
$\times$ & $\times$    & \checkmark & \checkmark & 5.75e-05 & 21.3 \\
$\times$ & \checkmark & \checkmark & \checkmark & 5.36e-05 & 22.02 \\
$\times$ & \checkmark & $\times$ & $\times$ & 7.51e-05 & 24.21 \\
\hline
\end{tabular}
\end{table}

\begin{table}[tb]
\centering
\caption{Ablation study for the Allen Cahn equation.}
\label{tab:ablation_allen}
\begin{tabular}{ccc|cc}
\hline
\multicolumn{3}{c|}{\textbf{Ablation Settings}} &
\multicolumn{2}{c}{\textbf{Performance}} \\
\hline
\textbf{Grad Norm} & \textbf{RWF} & \textbf{Causal Rel.} &
\textbf{L2 Error} & \textbf{Run Time (min)} \\
\hline
\checkmark & \checkmark & \checkmark & 2.782e-04 & 43.3 \\
\checkmark & $\times$ & $\times$ & 7.924e-05 & 42.6 \\
\checkmark & \checkmark & $\times$    & 1.475e-04 & 44.6 \\
\checkmark & $\times$    & \checkmark & 1.153e-04 & 40.8 \\
$\times$    & \checkmark & \checkmark & 5.197e-04 & 40.5 \\
\hline
\end{tabular}
\end{table}

\end{document}